\documentclass[]{mosi}
\usepackage{helvet} 
\usepackage{amsmath} 
\usepackage{natbib}
\usepackage{graphicx}
\usepackage{subcaption} 
\usepackage[toc,page,header]{appendix}
\usepackage[utf8]{inputenc} 
\usepackage[T1]{fontenc}    
\usepackage{hyperref}       
\usepackage{url}            
\usepackage{booktabs}       
\usepackage{amsfonts}       
\usepackage{nicefrac}       
\usepackage{microtype}      
\usepackage{wrapfig}
\usepackage{amssymb}        
\usepackage{titletoc}
\usepackage{minitoc}
\usepackage{array}
\usepackage{etoolbox}

\definecolor{lightblue}{RGB}{200, 230, 255}  
\definecolor{headerblue}{RGB}{150, 200, 255} 
\usepackage{pgfplots}
\usepackage{xcolor}
\usepackage{float}
\usepackage{comment}
\usepackage{multirow}
\usepackage{makecell}
\usepackage{siunitx}
\usepackage{tikz}
\usepackage{pgf-pie}
\usepackage[export]{adjustbox}
\usepackage{ragged2e}
\usepackage{tabularx}
\usepackage{caption}
\usepackage{enumitem}
\usepackage{pifont}
\usepackage[hang,flushmargin]{footmisc}
\usepackage{tcolorbox}
\tcbuselibrary{breakable}
\tcbuselibrary{skins}
\usepackage{tabularx}
\usepackage{listings}
\usepackage{threeparttable}
\usepackage{setspace}
\usepackage[scheme=plain]{ctex} 
\definecolor{MossCyan}{HTML}{82D9FF} 
\definecolor{MossBlue}{HTML}{82B1FF}

\definecolor{ForestGreen}{RGB}{34, 139, 34}
\definecolor{Red}{RGB}{255, 0, 0}

\definecolor{tickG}{rgb}{0.1, 0.588, 0.1}
\definecolor{crossR}{rgb}{0.588, 0.1, 0.1}

\definecolor{frenchblue}{rgb}{0.0, 0.45, 0.73}
\definecolor{babyblue}{rgb}{0.54, 0.81, 0.94}
\definecolor{classicrose}{rgb}{0.98, 0.8, 0.91}
\definecolor{beige}{rgb}{0.96, 0.96, 0.86}
\definecolor{forestgreen}{HTML}{2e7d43}
\definecolor{blue1}{HTML}{91BBE6}
\definecolor{blue2}{HTML}{3F90E0}
\definecolor{blue3}{HTML}{316FAD}
\definecolor{color1}{HTML}{FF9999}
\definecolor{color2}{HTML}{FF6666}
\definecolor{color3}{HTML}{FF3333}
\definecolor{color4}{HTML}{E60000}
\definecolor{color5}{HTML}{B30000}
\definecolor{color6}{HTML}{8CD98C}
\definecolor{color7}{HTML}{53c653}
\definecolor{color8}{HTML}{00B050}
\definecolor{color9}{HTML}{2d862d}
\definecolor{color10}{HTML}{206020}
\definecolor{color11}{HTML}{cca300}
\newtcolorbox{promptbox}[2][]{
    colback=white,
    coltext=black,
    arc=3mm,
    boxrule=0.5pt,
    colframe=black!60!white,
    title={#2},
    colbacktitle=black,
    coltitle=white,
    fonttitle=\bfseries,
    top=8pt,
    bottom=8pt,
    left=10pt,
    right=10pt,
    breakable,
    before upper={%
        \linespread{1}\selectfont
        \setlength{\parskip}{1ex plus 0.2ex minus 0.2ex}%
        \setlength{\parindent}{0pt}%
    },
    #1
}
\newtheorem{theorem}{Theorem}
\newtheorem{proposition}{Proposition}
\newtheorem{lemma}{Lemma}

\theoremstyle{definition}

\theoremstyle{remark}

\title{\scalebox{0.94}{BandPO: Bridging Trust Regions and Ratio Clipping via}\\\scalebox{0.94}{ Probability-Aware Bounds for LLM Reinforcement Learning}}
\author{
Yuan Li$^{1,2}$ \hspace{.3em}
Bo Wang$^{1}$\hspace{.3em}
Yufei Gao$^{1}$ \hspace{.1em}
Yuqian Yao$^{1,2}$ \hspace{.3em}
Xinyuan Wang$^{1}$ \hspace{.3em}
\\
Zhangyue Yin$^{1}$ \hspace{.2em}
Xipeng Qiu$^{1,2,\dagger}$
\\
[1ex]
\texttt{liyuan24@m.fudan.edu.cn}, \texttt{xpqiu@fudan.edu.cn} \\
[1ex]
$^{1}$Fudan University   
$^{2}$Shanghai Innovation Institute   
\\
}
\abstract{
\begin{abstract}
\noindent Proximal constraints are fundamental to the stability of the Large Language Model reinforcement learning. 
While the canonical clipping mechanism in PPO serves as an efficient surrogate for trust regions, we identify a critical bottleneck: fixed bounds strictly constrain the upward update margin of low-probability actions, disproportionately suppressing high-advantage tail strategies and inducing rapid entropy collapse. 
To address this, we introduce \textbf{Band-constrained Policy Optimization} (BandPO).
BandPO replaces canonical clipping with \textbf{Band}, a unified theoretical operator that projects trust regions defined by $f$-divergences into dynamic, probability-aware clipping intervals.
Theoretical analysis confirms that Band effectively resolves this exploration bottleneck.
We formulate this mapping as a convex optimization problem, guaranteeing a globally optimal numerical solution while deriving closed-form solutions for specific divergences.
Extensive experiments across diverse models and datasets demonstrate that BandPO consistently outperforms canonical clipping and Clip-Higher, while robustly mitigating entropy collapse.
\checkdata[GitHub]{\url{https://github.com/OpenMOSS/BandPO.git}}
\end{abstract}
}
\begin{document}
\maketitle
{
  \renewcommand{\thefootnote}{}
  \footnotetext{$^{\dagger}$ Corresponding author.}
}
\section{Introduction}
\label{sec:intro}
\noindent Reinforcement Learning from Human Feedback (RLHF) has established itself as the dominant paradigm for the post-training of Large Language Models (LLMs), wherein the proximal constraint on policy updates serves as a pivotal mechanism designed to balance optimization stability with effective exploration \citep{InstructGPT}.
\citet{PPO} emulate trust-region updates via clipping the surrogate objective, circumventing the expensive Fisher Information computations required by TRPO \citep{TRPO}.
This clipping mechanism has emerged as the default configuration for LLM RL and is extensively adopted in GRPO \citep{DeepSeek_Math} and its variants. 

\noindent While the canonical clipping mechanism ensures stability by emulating trust-region updates, \citet{Clip_Low_Clip_High} argue that it implicitly inhibits exploration by imposing a detrimental bias against policy entropy.
More precisely, in Reinforcement Learning with Verifiable Rewards (RLVR) scenarios, while lower-bound clipping tends to increase entropy, upper-bound clipping decreases it; with symmetric clipping thresholds, the latter effect dominates, resulting in net entropy reduction even when the algorithm is fed purely random rewards \citep{SpuriousRewards}.
This bias effectively silences the gradient signals for low-probability yet high-advantage actions, preventing the model from reinforcing novel, superior strategies that lie in the tail of the distribution. 
To mitigate this, \citet{DAPO} proposes the Clip-Higher strategy to relax the upper clipping bound. While \citet{EntropyMechanism} acknowledges that it delays entropy collapse, they also highlight its instability—often leading to performance collapse after saturation. This indicates that adjusting thresholds fails to address the inherent limitations of fixed clipping bounds.

\begin{figure}[t]
    \centering
    \begin{subfigure}[b]{0.48\linewidth}
        \centering
        \includegraphics[width=\linewidth]{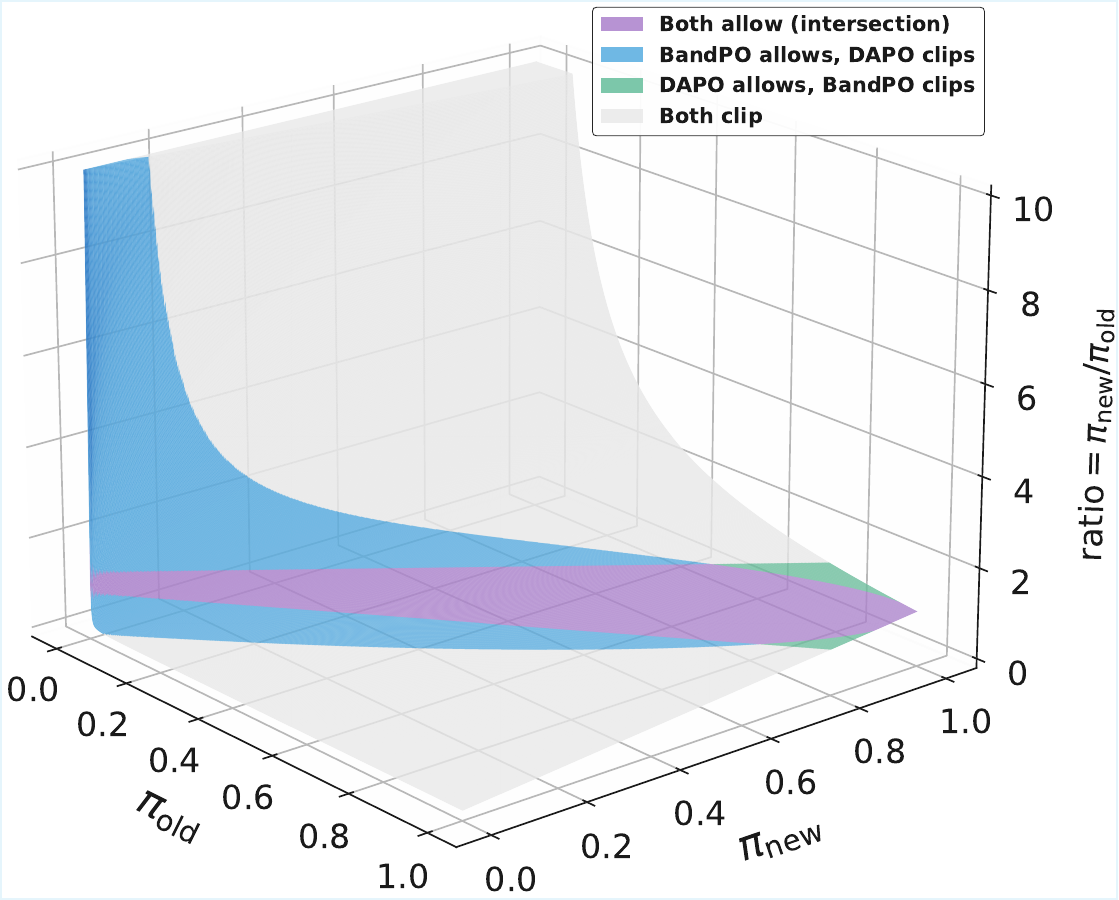}
        \caption{Ratio clipping regions}
        \label{fig:3d-vis-compare}
    \end{subfigure}
    \hfill 
    \begin{subfigure}[b]{0.48\linewidth}
        \centering
        \includegraphics[width=\linewidth]{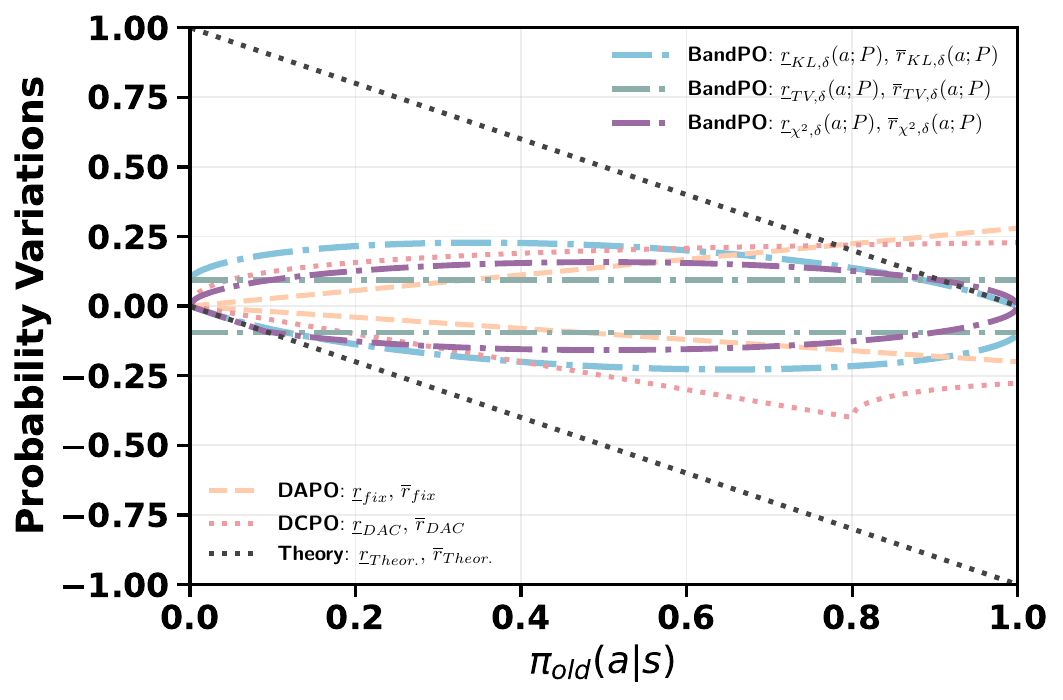}
        \caption{Bounds of Probability Variation}
        \label{fig:probability_variation}
    \end{subfigure}
    \vspace{-0.5em} 
    \caption{\textbf{Comparison of clipping bounds between BandPO and baselines.} 
    (\textbf{a}) \textbf{Comparison of ratio clipping regions: BandPO vs. DAPO.} While DAPO enforces fixed asymmetric bounds ($\epsilon_+ = 0.28, \epsilon_- = 0.2$), BandPO projects a KL-induced trust region ($\delta = 0.1$) into dynamic bounds. The blue region highlights the expanded margin for low-probability, positive-advantage actions, effectively preventing premature saturation and preserving critical exploration gradients. 
    (\textbf{b}) \textbf{Comparison of the Bounds of Probability Variation.} We visualize the bounds of variation derived from the Theoretical Simplex, DAPO, DCPO, and BandPO (ours). The symbols $\bar{r}$ and $\underline{r}$ denote the upper and lower clipping boundaries, respectively. Parameters are fixed at $\epsilon_+ = 0.28, \epsilon_- = 0.2$, and $\delta=0.1$. BandPO strictly adheres to physical simplex constraints while unlocking significant upward variation for low-probability actions.}
    \label{fig:combined_comparison}
\end{figure}

\noindent We formally characterize a critical structural bottleneck inherent in the canonical clipping mechanism, as detailed in Section~\ref{sec:bottlenecks}.
Specifically, constraining the probability ratio within fixed bounds enforces a linear dependence, wherein the maximum feasible probability variation scales proportionally with the old probability. 
Consequently, for positive-advantage actions, lower probabilities dictate vanishingly small margins for upward variation, rendering them susceptible to premature clipping and nullifying their gradient contributions.
Constrained by the necessity of proximal policy updates, the margin for expanding the fixed clipping ratio range is limited, failing to reconcile the fundamental tension between optimization constraints and effective exploration.

\noindent To address this bottleneck, we propose \textbf{Band-constrained Policy Optimization (BandPO)}. We introduce a unified theoretical operator, \textbf{Band}, which projects trust regions induced by general $f$-divergences into dynamic, probability-aware clipping intervals. 
BandPO substitutes the clipping in GRPO with Band, employing a single interpretable radius parameter to enforce proximal constraints, thereby significantly streamlining hyperparameter tuning.
Mathematically, we frame this mapping as a convex optimization problem, guaranteeing globally optimal numerical solutions while deriving efficient closed-form solutions for specific instances like Total Variation (TV) and Pearson $\chi^2$-divergence.
Crucially, we theoretically analyze the properties of Band and confirm that it naturally circumvents this bottleneck. As exemplified by the KL case in Figure~\ref{fig:3d-vis-compare}, the projected bounds adaptively expand the feasible upward margin for low-probability actions to prevent premature clipping, corresponding to the expanded blue region.
Empirically, compared with GRPO and GRPO with clip-higher, BandPO demonstrates consistent performance improvements across Qwen2.5 (3B, 7B) and Llama3 (8B) on multiple mathematical benchmarks, while robustly mitigating entropy collapse.
Our main contributions are summarized as follows:
\begin{itemize}[leftmargin=*, nosep]
    \item We formally characterize a critical bottleneck in canonical clipping, revealing that the feasible upward update margin scales linearly with the action probability. This tends to nullify gradients for low-probability, positive-advantage actions, inhibiting effective exploration.
    \item We propose BandPO, introducing a unified \textbf{Band} operator to project $f$-divergence-induced trust regions into dynamic clipping intervals. We formulate this as a convex optimization problem, guaranteeing globally optimal numerical solutions while deriving closed-form solutions for specific divergences, theoretically demonstrating that BandPO circumvents the bottleneck.
    \item We demonstrate that BandPO achieves consistent performance gains over GRPO and Clip-Higher baselines across various models on math benchmarks, robustly mitigating entropy collapse.
\end{itemize}

\section{Related Work}
\paragraph{From Trust Regions to Ratio Clipping.} 
Proximal constraints are fundamental to stable policy optimization, ensuring that the updated policy remains within a controlled vicinity of the sampling policy. 
Introduced by \citet{TRPO}, the trust-region concept has been widely adopted in policy gradient optimization to impose proximal constraints. 
\citet{TRPO} employs the KL-induced trust region to ensure that the new policy remains within a small neighborhood of the old policy at each update.
Theoretically, such trust regions can be characterized by various distributional discrepancies, including integral probability metrics~\citep{IPM} and the family of $f$-divergences—such as Pearson $\chi^2$ and TV~\citep{f-divergence, f-gan}.
However, the resulting constrained optimization problem involves an inequality constraint, rendering it computationally prohibitive for large-scale applications.
To address this issue, \citet{PPO} proposed two variants—PPO-Penalty and PPO-Clip—that eliminate the need for conjugate gradient methods.
The PPO-Clip mechanism and its variants have been extensively adopted across diverse domains, spanning complex strategic gaming \citep{OpenAIFive, HideAndSeek}, physics-based character animation \citep{DeepMimic}, and robotic control \citep{DexterousHand, MAPPO}.
Subsequent works have established the clipping mechanism as the dominant paradigm for LLM post-training \citep{InstructGPT, RLHF_Frame}. 
Concurrently, there is a notable shift towards critic-free paradigms to enhance computational efficiency \citep{ReMax, RLOO, REINFORCE_plus_plus, DeepSeek_Math}.

\paragraph{Adaptive Clipping Variants.}
The clipping mechanism is favored for its simplicity, yet it also raises ongoing questions regarding the appropriate choice of clipping bounds. 
To mitigate instability from negative-advantage actions, \citet{Dual_Clip} introduces an auxiliary lower bound, adopted in the LLM RL framework~\citep{VERL}. 
\citet{PPO_Lamda} proposes a state-wise adaptive clipping mechanism that modulates the clipping strength according to state importance, while \citet{Decaying_Clip} applies a simple time-decaying schedule to the clipping range. 
Despite their empirical success, these heuristics rely on auxiliary hyperparameters lacking a clear theoretical grounding, rendering them brittle and difficult to tune.
\citet{TRGPPO} establishes a theoretical connection between KL-divergence and the clipping bounds, demonstrating performance improvements in continuous control tasks.

\paragraph{Clip Control in LLM.}
LLM RL operates in an extremely high-dimensional action space, where the combination of long reasoning horizons and extensive group sampling results in a high cumulative incidence of low-probability actions. To address this issue, DAPO~\citep{DAPO}  proposes the Clip-Higher strategy, which decouples and relaxes the clipping upper bounds. 
Building upon the decoupled bounds, DCPO~\citep{DCPO} derives dynamic clipping boundary functions via inequality relaxation, which dynamically adjust according to the action probabilities.

However, compared to the maturity of continuous control, the theoretical underpinnings of clipping mechanisms in LLM RL remain under-explored.
Existing methods lack a principled framework to control clipping bounds via simple, effective, and interpretable parameters, consequently struggling to balance proximal constraints with effective exploration.
To bridge this gap, we propose BandPO, which employs a unified operator that projects $f$-divergences-induced trust regions into probability-aware clipping intervals, resolving the bottleneck using an interpretable parameter.

\section{Preliminaries}
\label{sec:preliminaries}
\subsection{Notation}
\label{subsec:notation}
We formulate the RL alignment of LLMs as a discrete Markov Decision Process. Let $\pi_\theta$ denote the policy represented by the LLM.
Given a prompt $x$ sampled from a dataset $\mathcal{D}$, the policy generates a response sequence $y = (a_1, a_2, \dots, a_T)$ by auto-regressively sampling from a vocabulary $\mathcal{V}$ of size $V = |\mathcal{V}|$. Each action $a_t \in \mathcal{V}$ corresponds to a token generated by the LLM. 
At step $t$, the state $s_t = (x, y_{<t})$ comprises the prompt $x$ and the preceding tokens $y_{<t} = (a_1, \dots, a_{t-1})$. 
The policy $\pi_\theta$ maps $s_t$ to a conditional probability distribution $\pi_\theta(\cdot \mid s_t)$ over $\mathcal{V}$. 
We denote by $\mathbb{R}$ and $\mathbb{R}_{+}$ the sets of real and non-negative real numbers, respectively. Let $\Delta^{V} \triangleq \{p \in \mathbb{R}_{+}^{V} \mid \sum_{i=1}^{V} p_i = 1\}$ denote the probability simplex over $V$ categories.
The optimization objective is to maximize the expected reward: $J(\theta) = \mathbb{E}_{x \sim \mathcal{D}, y \sim \pi_\theta}[R(x, y)]$, where $R(x, y)$ denotes a sparse, sequence-level scalar reward, typically derived from verification signals from verification signals.

\subsection{Clip-Based Proximal Constraints in LLM RL}
\label{subsec:clipRL}

Consider an iterative optimization process where we update $\pi_\theta$ using trajectories sampled by $\pi_{\text{old}}$. The probability ratio $r_t(\theta) = \frac{\pi_\theta(a_t|s_t)}{\pi_{\text{old}}(a_t|s_t)}$ serves as an importance sampling weight to correct for the distributional shift. 
\citet{PPO} proposed clipping $r_t$ to impose proximal constraints, while employing a critic model for Generalized Advantage Estimation \citep{GAE}.
Inheriting the clipping mechanism, \citet{DeepSeek_Math} introduced Group Relative Policy Optimization (GRPO), which circumvents the computational burden of training a critic model by estimating advantages from a group of sampled responses. 
We denote the objective as $\mathcal{J}^{\text{GRPO}}(\theta)$, which aggregates the per-token objectives across a group of $G$ outputs:
\begin{equation}
\label{eq:grpo}
\mathcal{J}^{\text{GRPO}}(\theta) = \mathbb{E}_{x \sim \mathcal{D}, \{y_i\}_{i=1}^G \sim \pi_{\text{old}}} \left[ \frac{1}{G} \sum_{i=1}^G {\frac{1}{T_i} \sum_{t=1}^{T_i} \mathcal{J}_t(\theta; y_i)} \right],
\end{equation}
Specifically, we formulate $\mathcal{J}_t(\theta; y_i)$ to admit asymmetric clipping boundaries:
\begin{equation}
\label{eq:token_objective}
    \mathcal{J}_t(\theta; y_i) = \min \big( r_{t,i} A_{t,i}, \text{clip}(r_{t,i}, 1-\epsilon_-, 1+\epsilon_+) A_{t,i} \big) - \beta \underbrace{D_{\mathrm{KL}}(\pi_\theta(\cdot|s_t) \| \pi_{\text{ref}}(\cdot|s_t))}_{\text{per-token KL divergence}}.
\end{equation}
This formulation averages the objective over tokens to accommodate varying sequence lengths $T_i$. $G$ denotes the group size, and $r_{t,i}$ represents the probability ratio. The advantage $A_{t,i}$ is derived from the sequence-level rewards $R_i$, standardized within the group via $A_{t,i} = (R_i - \mu_R)/\sigma_R$, where $\mu_R$ and $\sigma_R$ are the mean and standard deviation of rewards $\{R_j\}_{j=1}^G$, respectively.
The term $\beta D_{\mathrm{KL}}$ serves as a regularization penalty, anchoring $\pi_\theta$ to the reference policy $\pi_{\text{ref}}$ to preserve linguistic coherence. 
While GRPO employs symmetric bounds ($\epsilon_+ = \epsilon_-$), Equation~\eqref{eq:token_objective} generalizes this to allow asymmetric bounds ($\epsilon_+ > \epsilon_-$).

\section{The Bottleneck in Canonical Clipping}
\label{sec:bottlenecks}
The canonical clipping mechanism confines the probability ratio $r_t(\theta)$ to the interval $[1-\epsilon_-, 1+\epsilon_+]$, where $\epsilon_-, \epsilon_+ > 0$ are fixed constants. This constraint is formulated as:
\begin{equation}
    \label{eq:fixed_ratio_range}
    1 - \epsilon_{-} \leq \frac{\pi_{\theta}(a \mid s)}{\pi_{\text{old}}(a \mid s)} \leq 1 + \epsilon_{+}.
\end{equation}
We define the probability variation of an action $a$ given state $s$ as: $\Delta \pi(a|s) = \pi_{\theta}(a|s) - \pi_{\text{old}}(a|s)$, which explicitly quantifies the deviation of the probability relative to the sampling policy.
Multiplying all terms in Inequality~\eqref{eq:fixed_ratio_range} by the strictly positive $\pi_{\text{old}}(a|s)$ and subtracting $\pi_{\text{old}}(a|s)$ yields the set of feasible probability variations, denoted as $\mathcal{C}_{\text{clip}}$:
\begin{equation}
    \label{eq:fixed_variation_range_set}
    \left\{ \Delta \pi(a|s) \;\middle|\; -\epsilon_- \pi_{\text{old}}(a|s) \le \Delta \pi(a|s) \le \epsilon_+ \pi_{\text{old}}(a|s) \right\}.
\end{equation}
Theoretically, the constraint $\pi_\theta(\cdot|s) \in \Delta^V$  allows the variation to range from $-\pi_{\text{old}}$ (dropping to $0$) to $1 - \pi_{\text{old}}$ (rising to $1$). We define this maximal feasible set as $\mathcal{C}_{\text{simplex}}$:
\begin{equation}
    \label{eq:theoretical_variation_range_set}
    \left\{ \Delta \pi(a|s) \;\middle|\; -\pi_{\text{old}}(a|s) \le \Delta \pi(a|s) \le 1 - \pi_{\text{old}}(a|s) \right\}.
\end{equation}
The derivation above establishes the relationship $\Delta \pi(a|s) = (r_t(\theta) - 1)\pi_{\text{old}}(a|s)$.
Using this relationship to map $\mathcal{C}_{\text{simplex}}$ into the ratio space yields the theoretical bounds of $r_t(\theta)$:
\begin{equation}
\label{eq:theory_ratio_range}
0 \le \frac{\pi_{\theta}(a \mid s)}{\pi_{\text{old}}(a \mid s)} \le \frac{1}{\pi_{\text{old}}(a \mid s)}.
\end{equation}

As illustrated in Fig. \ref{fig:probability_variation}, analysis within $\mathcal{C}_{\text{simplex}}$ reveals that fixed clipping bounds on $r_t(\theta)$ (e.g., DAPO) constrain probability variations to scale linearly with $\pi_{\text{old}}(a|s)$. Consequently, the feasible upward shift vanishes as the probability approaches zero, contradicting the theoretical upper bound. This discrepancy induces premature clipping for low-probability actions with positive advantages, effectively nullifying gradient contributions. Conversely, in high-probability regimes, the upper bounds of DAPO and DCPO exceed inherent simplex limits, rendering the constraint mathematically vacuous. Thus, fixed bounds fail to reconcile proximal constraints with exploration efficacy, while dynamic bounds lacking clear theoretical grounding result in partial constraint failure. These limitations necessitate a transition to theoretically grounded, dynamic clipping bounds of the probability ratio.

Consider a tail token $\pi_{\text{old}} = 0.08$ with $\epsilon_+ = 0.2$. The allowed update $\Delta \pi \le 0.016$ is negligible compared to the theoretical capacity ($0.92$). Enabling meaningful exploration (e.g., $\Delta \pi \approx 0.4$) would require raising $\epsilon_+$ to $\approx 5.0$. However, this renders the constraint mathematically vacuous for head tokens ($\pi_{\text{old}} = 0.8$), as the permitted shift $\Delta \pi \le 4.0$ far exceeds the physical limit of $0.2$.
Thus, the static clipping approach is fundamentally ill-equipped to bridge the gap between stable optimization and the acquisition of novel and effective strategies from the tail of the probability distribution. 

\section{Method}
\label{sec:method}
In this section, we introduce \textbf{BandPO}. Central to BandPO is the \textbf{Band} operator, which projects trust regions defined by $f$-divergences into probability-aware clipping intervals. Our analysis of Band's theoretical properties demonstrates that it effectively resolves the exploration bottleneck in Section~\ref{sec:bottlenecks}.

\subsection{$f$-Divergence-Induced Trust Regions}
\label{subsec:f_trust_region}
Consider a fixed state $s_t$. Let $P(\cdot) \triangleq \pi_{\text{old}}(\cdot|s_t) \in \Delta^{V}$ and $Q(\cdot) \triangleq \pi_{\theta}(\cdot|s_t) \in \Delta^{V}$ denote the distributions over $\mathcal{V}$, respectively.
We formalize the trust region using the family of $f$-divergences \citep{f-divergence}. Let $f: \mathbb{R}_{+} \to \mathbb{R}$ be a strictly convex function with $f(1) = 0$. The divergence $D_f(Q \| P)$ is defined as:
\begin{equation}
\label{eq:f_divergence_def}
D_f(Q \| P) \triangleq \sum_{a \in \mathcal{V}} P(a) f \left( \frac{Q(a)}{P(a)} \right).
\end{equation}
With $P$ serving as the anchor, we define the trust region $\mathcal{T}_{f,\delta}(P)$ as the convex set constraining $Q$ within a $\delta$-neighborhood of $P$ on the probability simplex $\Delta^V$:
\begin{equation}
\label{eq:f_trust_region_set}
\mathcal{T}_{f,\delta}(P) \triangleq \left\{ Q \in \Delta^{V} \mid D_f(Q \| P) \le \delta \right\},
\end{equation}
where $\delta > 0$ dictates the radius of the trust region.
This geometric formulation generalizes standard proximal constraints.
Notably, choosing the generator function $f(u) = -\log u + u - 1$ recovers the specific trust region employed in TRPO \citep{TRPO}, i.e., $D_{\mathrm{KL}}(P \| Q) \le \delta$.

\subsection{Band: An Operator for Projecting Trust Regions to Probability-Aware Clipping Bounds}
\label{subsec:band_def}
We now derive the probability-aware ratio bounds induced by the geometric constraint $\mathcal{T}_{f,\delta}(P)$.
To facilitate the analysis, we recast the probability ratio $r_t(\theta)$ as a function of the candidate distribution $Q \in \Delta^V$.
For any token $a \in \mathcal{V}$ satisfying $P(a) > 0$, the ratio function is redefined as:
\begin{equation}
\label{eq:ratio_func_def}
r(a; Q, P) \triangleq \frac{Q(a)}{P(a)}.
\end{equation}
While $P$ remains fixed, the candidate distribution $Q \in \Delta^V$ varies subject to the constraint $\mathcal{T}_{f,\delta}(P)$.

\paragraph{Optimal Dynamic Bounds.}
The upper bound of the probability ratio is determined by maximizing the probability $Q(a)$ subject to the constraint $\mathcal{T}_{f,\delta}(P)$.
Since $\mathcal{T}_{f,\delta}(P)$ forms a convex feasible set, this formulation constitutes a convex optimization problem with a linear objective function with respect to the variable $Q$:
\begin{equation}
\label{eq:upper_bound_opt}
{
\overline{r}_{f,\delta}(a; P)
\triangleq
\max_{Q \in \mathcal{T}_{f,\delta}(P)} \; \frac{Q(a)}{P(a)}.
}
\end{equation}
Symmetrically, the minimal admissible probability ratio is determined by solving:
\begin{equation}
{
\label{eq:lower_bound_opt}
\underline{r}_{f,\delta}(a; P)
\triangleq
\min_{Q \in \mathcal{T}_{f,\delta}(P)} \; \frac{Q(a)}{P(a)}.
}
\end{equation}
Crucially, Problems~\eqref{eq:upper_bound_opt} and~\eqref{eq:lower_bound_opt} are convex programs. 
Under strictly convex $f$ and $P(a)\in(0,1)$, any local optimum is the global optimum, enabling stable numerical solutions.

\paragraph{The Band Operator.}
We propose \textbf{Band}, a unified operator designed to supersede canonical clipping by strictly constraining the probability ratio within the feasible interval induced by the $f$-divergence trust region.
Given a generator function $f$ and radius $\delta$, for $\forall P(a) \in (0,1)$, solving Problems~\eqref{eq:upper_bound_opt} and~\eqref{eq:lower_bound_opt} yields the rigorous lower and upper bounds for the ratio.
Crucially, this derivation projects the high-dimensional geometric constraint $\mathcal{T}_{f,\delta}(P)$ onto a scalar interval specific to action $a$, to which the ratio is strictly confined.
We formulate this operation as:
\begin{equation}
\label{eq:band_operator}
{
\mathrm{Band}_{f,\delta}\big(r;\,a,P\big)
\triangleq
\operatorname{clip}\!\left(
r, \; \underline{r}_{f,\delta}(a; P), \; \overline{r}_{f,\delta}(a; P)
\right).
}
\end{equation}
In stark contrast to the fixed clipping interval $[1-\epsilon, 1+\epsilon]$, Band yields probability-aware bounds governed solely by the single, interpretable trust-region radius $\delta$.

\subsection{Reduction to Univariate Optimization}
\label{subsec:reduction}
Although Problems~\eqref{eq:upper_bound_opt} and~\eqref{eq:lower_bound_opt} are convex, the decision variable $Q$ resides in a high-dimensional simplex ($V \approx 10^5$), rendering direct optimization computationally prohibitive.
We circumvent this by exploiting the structural symmetry of the divergence constraint to strictly reduce the problem to a univariate formulation, as established in the following lemma, proven in Appendix ~\ref{app:proof_lemma_scaling}.
\begin{lemma}[Optimality of Uniform Complement Rescaling]
\label{lem:complement_rescaling}
Given a reference distribution $P \in \Delta^V$ with full support (i.e., $P(v)>0, \forall v$) and an action $a \in \mathcal{V}$, the optimal solution $Q^\star$ to the extremal Problems~\eqref{eq:upper_bound_opt} and~\eqref{eq:lower_bound_opt} {must strictly preserve} the relative probability proportions within the complement set $\mathcal{V} \setminus \{a\}$. Specifically, the probability ratio is constant for all non-target actions:
\begin{equation}
\label{eq:q_star_structure}
\frac{Q^\star(b)}{P(b)} = c, \quad \forall b \in \mathcal{V} \setminus \{a\},
\end{equation}
where the scaling factor $c \in \mathbb{R}_{+}$ is uniquely determined by the simplex normalization constraint $\sum_{v \in \mathcal{V}} Q^\star(v) = 1$. 
This implies that $Q^\star$ is fully parameterized by the single scalar ratio $r = Q^\star(a)/P(a)$ at the target action. 
\end{lemma}

Based on Lemma~\ref{lem:complement_rescaling}, the optimization reduces to determining the scalar ratio $r \triangleq Q(a)/P(a)$ in the target action $a$. Let $p \triangleq P(a)$. The simplex constraint $\sum_{v} Q(v) = 1$ uniquely determines the complement scaling factor $c$ as a function of $r$, since $rp + c(1-p) = 1 \implies c(r) = \frac{1 - rp}{1 - p}$. Substituting this mapping into Equation~\eqref{eq:f_divergence_def}, the divergence collapses into a univariate function $g_f(p, r)$:
\begin{equation}
\label{eq:scalar_g_derivation}
D_f(Q \| P) = \underbrace{P(a) f(r)}_{\text{target}} + \sum_{b \neq a} P(b) f(c) = p f(r) + (1-p) f\left( \frac{1 - rp}{1 - p} \right)\triangleq g_f(p, r).
\end{equation}
Equation~\eqref{eq:scalar_g_derivation} strictly reduces the high-dimensional constraint $D_f(Q \| P) \le \delta$ to a scalar inequality $g_f(p, r) \le \delta$.
This scalarization effectively shifts the optimization domain from the high-dimensional simplex $Q \in \Delta^V$ to the univariate scalar $r$. Consequently, we transform Problems~\eqref{eq:upper_bound_opt} and~\eqref{eq:lower_bound_opt} into a tractable root-finding problem, where the distinct roots of the equation $g_f(p, r) = \delta$ yield the precise values of the projection clipping bounds $\underline{r}_{f,\delta}$ and $\overline{r}_{f,\delta}$, as formalized in the following theorem (proof in Appendix~\ref{app:proof_thm_scalar}).
\begin{theorem}[Exact Scalarization of Trust-Region Constraints]
\label{thm:scalar_reduction}
Assume the generator function $f: \mathbb{R}_+ \to \mathbb{R}$ is strictly convex with $f(1)=0$. For any action $a$ with $P(a) = p \in (0,1)$, the Problems~\eqref{eq:upper_bound_opt} and~\eqref{eq:lower_bound_opt} are equivalent to finding the roots of the scalar function:
\begin{equation}
\label{eq:scalar_g_def}
g_f(p, r) \triangleq p f(r) + (1 - p) f\left( \frac{1 - r p}{1 - p} \right) = \delta,
\end{equation}
subject to the feasibility constraint $r \in [0, 1/p]$.
Specifically, $g_f(p, r)$ is strictly convex with respect to $r$ and achieves its global minimum of $0$ at $r=1$. Consequently, for any valid trust-region radius $\delta > 0$, the equation $g_f(p, r) = \delta$ possesses exactly two unique roots, corresponding to the optimal clipping bounds:
\begin{align}
    \underline{r}_{f,\delta}(a; P)
    &= \min \left\{ r \in [0, 1] \mid g_f(p, r) = \delta \right\},
    \label{eq:r_start_lower} \\
    \overline{r}_{f,\delta}(a; P)
    &= \max \left\{ r \in [1, 1/p] \mid g_f(p, r) = \delta \right\}.
    \label{eq:r_start_upper}
\end{align}
\end{theorem}

\subsection{Properties of Band Bounds}
\label{subsec:properties}
\begin{figure}[t]
    \centering
    \includegraphics[width=0.55\linewidth]{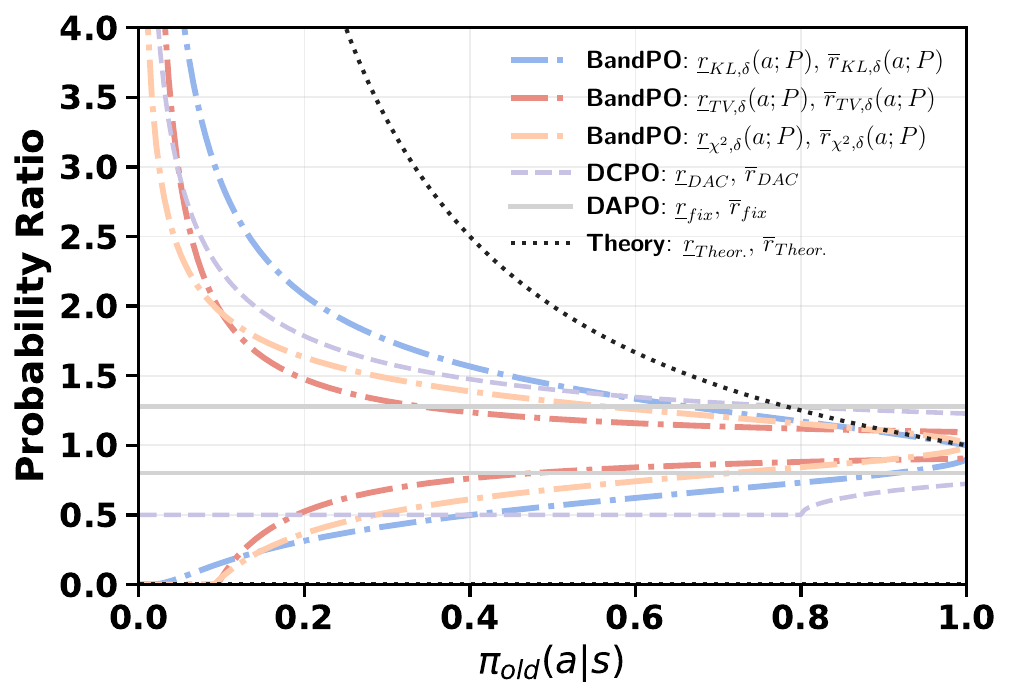}
    \vspace{-6pt} 
    \caption{\textbf{Comparison of Probability Ratio Bounds.} We visualize the ratio bounds derived from the Theoretical Simplex, DAPO, DCPO, and BandPO (ours). As $p \to 1$, BandPO bounds exhibit strict monotonicity: upper bounds decrease while lower bounds increase towards 1. Conversely, as $p \to 0$, the upper bounds of both DCPO and BandPO expand rapidly, effectively preventing premature clipping. Note that for TV and $\chi^2$, the radius $\delta=0.1$ triggers the simplex saturation condition, where the lower bounds explicitly align with the theoretical limit of $0$.}
    \vspace{-3pt} 
    \label{fig:probability_ratio}
\end{figure}

Building upon Theorem~\ref{thm:scalar_reduction}, we theoretically analyze the asymptotic behavior and strict monotonicity of the bounds $\overline{r}_{f,\delta}(p)$ and $\underline{r}_{f,\delta}(p)$ with respect to $p \in (0, 1)$, with proofs provided in Appendix~\ref{app:proof_prop_limits} and~\ref{app:proof_prop_monotonicity}.
This analysis demonstrates that Band theoretically resolves the exploration bottlenecks identified in Section~\ref{sec:bottlenecks}.
\begin{proposition}[Asymptotic Behavior of Band Bounds]
\label{prop:limits}
Given a trust-region radius $\delta > 0$, the bounds $\overline{r}_{f,\delta}(p)$ and $\underline{r}_{f,\delta}(p)$ exhibit the following limiting behaviors as the probability $p \in (0, 1)$ approaches the boundaries of the simplex:
\begin{equation}
    \lim_{p\to 0^+}\overline{r}_{f,\delta}(p) = +\infty, \quad
    \lim_{p\to 0^+}\underline{r}_{f,\delta}(p) = 0, \quad
    \lim_{p\to 1^-}\overline{r}_{f,\delta}(p) = 1.
    \label{eq:all_limits}
\end{equation}
\end{proposition}
\begin{proposition}[Strict Monotonicity of Band Bounds]
\label{prop:monotonicity}
The clipping bounds are strictly monotonic functions of the probability $p$. Specifically, given a $\delta > 0$, the upper bound $\overline{r}_{f,\delta}(p)$ is strictly decreasing with respect to $p$, while the lower bound $\underline{r}_{f,\delta}(p)$ is strictly increasing with respect to $p$.
\end{proposition}
\paragraph{Resolving the Exploration Bottleneck.}
To demonstrate the constraint behaviors, Figures~\ref{fig:probability_variation} and~\ref{fig:probability_ratio} visualize the bounds for probability variation and ratios, respectively. We contrast the Theoretical Simplex, DAPO, and DCPO against BandPO, instantiated with KL, Pearson $\chi^2$, and TV divergences.
As illustrated, the ratio upper bounds of probability-aware methods (DCPO and BandPO) decrease, while the lower bounds increase with respect to $p$.
Furthermore, their upper ratio bounds expand rapidly as $p \to 0$, which prevents the premature suppression of low-probability actions with positive advantages, thereby preserving effective gradients and facilitating the exploration of valuable tail strategies.
For $\mathrm{Band}_{f,0.1}$, the KL variant permits a broader ratio range compared to the TV and $\chi^2$ variants.
Notably, TV and $\chi^2$ correctly trigger the simplex saturation condition at low probabilities, where the lower bound explicitly aligns with the theoretical simplex limit of $0$.
Crucially, all BandPO variants maintain strict geometric consistency with the probability simplex $\Delta^V$, underscoring the method's theoretical soundness. 
In contrast, heuristic approaches like DCPO and DAPO violate the theoretical upper bound in high-probability regimes.
Ultimately, while BandPO and DCPO successfully mitigate the exploration bottleneck, BandPO offers a more rigorous foundation, guaranteeing valid constraints governed by an interpretable parameter $\delta$.

\paragraph{Resolving the Exploration Bottleneck.}
Figures~\ref{fig:probability_variation} and~\ref{fig:probability_ratio} compare BandPO (KL, $\chi^2$, TV) against baselines. Both BandPO and DCPO expand ratio bounds as $p \to 0$, preventing premature suppression of tail actions to facilitate exploration. However, unlike heuristics (DCPO, DAPO) that violate theoretical limits at high probabilities, all BandPO variants maintain strict geometric consistency with the simplex $\Delta^V$. Specifically, BandPO-KL yields the broadest range, while TV and $\chi^2$ correctly capture the simplex saturation at $0$. Thus, BandPO uniquely resolves the exploration bottleneck with a rigorous foundation, guaranteeing valid constraints unlike heuristic approximations.

\subsection{Solving Band Bounds}
\label{subsec:solving}
For a chosen function $f$ and radius $\delta$, given the probability $p = P(a)$, we solve Problems~\eqref{eq:r_start_lower} and~\eqref{eq:r_start_upper} to derive the probability ratio bounds $\underline{r}_{f,\delta}(a; P)$ and $\overline{r}_{f,\delta}(a; P)$.

\paragraph{Simplex Saturation.}
For a sufficiently large radius $\delta$, the divergence constraint may extend beyond the boundaries of the probability simplex $\Delta^V$. This creates a conflict between the simplex boundaries and the divergence constraint, rendering Problems~\eqref{eq:r_start_lower} and~\eqref{eq:r_start_upper} ill-defined. We term this phenomenon \textbf{Simplex Saturation}. To enforce the simplex constraint, we formalize this saturation condition and align the Band bounds with the simplex limits, with details deferred to Appendix~\ref{app:proof_prop_saturation}.
\begin{proposition}[Constraint Saturation]
\label{prop:saturation}
Based on Inequality~\eqref{eq:theory_ratio_range}, denote simplex bounds as $r_{\max} \triangleq 1/p$ and $r_{\min} \triangleq 0$.
The optimal Band upper bound $\overline{r}_{f,\delta}(p)$ is given by:
\begin{equation}
\overline{r}_{f,\delta}(p) =
\begin{cases}
r_{\max}, & \text{if } g_f(p, r_{\max}) \le \delta, \\
r^{\dagger}, & \text{otherwise,}
\end{cases}
\end{equation}
where $r^{\dagger}$ denotes the unique root of $g_f(p, r) = \delta$ in $(1, r_{\max})$, which is guaranteed to exist in the non-saturated regime.
Symmetrically, $\underline{r}_{f,\delta}(p) = 0$ if $g_f(p, r_{\min}) \le \delta$; otherwise, it is the unique root of $g_f(p, r) = \delta$ in $(0, 1)$.
\end{proposition}
\textbf{Generic Numerical Solver.}
In the active regime, the optimal bounds correspond strictly to the unique roots of the binding equation $g_f(p, r) = \delta$.
These roots are well-separated by the singularity at $r=1$.
Consequently, they can be efficiently computed via standard bracketed root-finding algorithms (e.g., Bisection or Brent's method), which are guaranteed to converge to the global optimum.
We provide the rigorous convergence analysis and the standard solver formulation in Appendix~\ref{app:numerical_solver}, along with a concrete instantiation for the KL-divergence.
\textbf{Closed-Form Solutions.}
Specific $f$-divergences admit closed-form solutions to Problems~\eqref{eq:r_start_lower} and~\eqref{eq:r_start_upper}, offering superior computational efficiency compared to numerical methods. We provide bounds for TV and Pearson $\chi^2$ divergence below, with derivations detailed in Appendix~\ref{app:proof_prop_closedform}.
\begin{proposition}[Closed-Form Band Bounds for TV and Pearson $\chi^2$]
\label{prop:closedform}
Consider a reference probability $p \in (0,1)$ and a trust-region radius $\delta > 0$.
Defining the generator functions as $f_{\mathrm{TV}}(u)=\frac{1}{2}\lvert u-1\rvert$ and $f_{\chi^2}(u)=(u-1)^2$, the closed-form clipping bounds are derived as follows.
For TV, the bounds scale linearly with the inverse probability:
\begin{equation}
\label{eq:tv_bounds}
    \overline{r}_{\mathrm{TV},\delta}(p) = 1 + \frac{\delta}{p},
    \qquad
    \underline{r}_{\mathrm{TV},\delta}(p) = 1 - \frac{\delta}{p}.
\end{equation}

For Pearson $\chi^2$, the bounds scale with the inverse square root of the odds:
\begin{equation}
\label{eq:chi2_bounds}
    \overline{r}_{\chi^2,\delta}(p) = 1 + \sqrt{\frac{\delta (1-p)}{p}}, \;
    \underline{r}_{\chi^2,\delta}(p) = 1 - \sqrt{\frac{\delta (1-p)}{p}}.
\end{equation}
\end{proposition}

\subsection{BandPO: Band-Constrained Policy Optimization}
\label{subsec:bandpo}
We propose \textbf{BandPO}, a policy optimization framework that substitutes the canonical clipping mechanism of GRPO with the theoretically rigorous \textbf{Band} operator. 
Formally, BandPO maximizes the objective $\mathcal{J}_{\text{BandPO}}(\theta)$:
\begin{equation}
\label{eq:bandpo_full}
 \mathcal{J}_{\text{BandPO}}(\theta) = \mathbb{E}_{x \sim \mathcal{D}, \{y_i\}_{i=1}^G \sim \pi_{\text{old}}} \left[ \frac{1}{G} \sum_{i=1}^G \frac{1}{T_i} \sum_{t=1}^{T_i} \mathcal{J}^{\text{Band}}_t(\theta; y_i) \right].
\end{equation}
By employing the probability-aware clipping interval derived in Equation~\eqref{eq:band_operator}, the core per-token surrogate objective $\mathcal{J}^{\text{Band}}_t(\theta; y_i)$ is formulated as:
\begin{equation}
\label{eq:bandpo_token_loss}
\begin{split}
    \mathcal{J}^{\text{Band}}_t(\theta; y_i) = \min \big( r_{t,i} A_{t,i}, {\mathrm{Band}_{f,\delta}\big(r_{t,i};\,y_{t,i}, {\pi_{\text{old}}(\cdot|s_{t,i})}\big)} A_{t,i} \big) - \beta D_{\mathrm{KL}}(\pi_{\text{ref}} \| \pi_\theta)_t,
\end{split}
\end{equation}
where $r_{t,i} = \frac{\pi_\theta(y_{t,i}|s_{t,i})}{\pi_{\text{old}}(y_{t,i}|s_{t,i})}$ and the advantage $A_{t,i}$ is computed as described in Section~\ref{subsec:clipRL}.
By projecting the trust region induced by $f$-divergence into a probability-aware scalar interval specific to the target action $y_{t,i}$, the Band operator strictly confines the ratio within a theoretically feasible domain, thereby balancing optimization stability with the effective exploration of tail strategies.

\section{Empirical Study}
\subsection{Experimental Setup}
\textbf{Models and Datasets.} We employ a composite training set merging DAPO~\citep{DAPO} with MATH Levels 3--5~\citep{MATH} to fine-tune four models spanning diverse architectures and scales: Qwen2.5-3B-Instruct~\citep{Qwen2_5} and the DeepSeek-R1-Distill family (Qwen-1.5B/7B, Llama-8B)~\citep{DeepSeek_R1}. To evaluate reasoning robustness across varying difficulty levels, we validate on AMC 2023~\citep{AMC2023}, AIME 2024~\citep{MAA2024}, and AIME 2025~\citep{MAA2025}.

\textbf{Implementation Details.} We conduct all experiments on $8\times$ NVIDIA H200 GPUs using the \texttt{verl} framework~\citep{VERL}. We train models for 800 steps (for 1.5B/3B) or 500 steps (for 7B/8B), utilizing a global batch size of 256 (mini-batch 64, micro-batch 8) and a learning rate of $1 \times 10^{-6}$. During generation, we set both the sampling temperature and top-$p$ to $1.0$.
All experiments are repeated three times to ensure statistical significance.
Regarding clipping bounds, we adopt asymmetric thresholds with $\epsilon_+ = 0.28$ and $\epsilon_- = 0.2$. For BandPO, we enforce the KL divergence as the trust region constraint with $\delta=0.05$. Crucially, we implement a CUDA-accelerated parallel bisection method to efficiently solve for the Band bounds.

\textbf{Baselines and Metrics.} 
To isolate the impact of clipping mechanisms, we implement comparisons within the GRPO framework~\citep{DeepSeek_Math}. We establish two primary baselines: (1) \textbf{GRPO}, representing canonical symmetric clipping~\citep{PPO}; and (2) \textbf{GRPO w/ Clip-Higher}, representing the state-of-the-art asymmetric clipping introduced by DAPO~\citep{DAPO}. 
For our method (BandPO), (3) \textbf{GRPO w/ $\text{Band}_{\text{KL},\delta}$}, we replace the clip operator with the Band operator. 
Evaluation metrics include \textbf{pass@32} to measure peak reasoning capability (probability of at least one correct solution) and \textbf{mean@32} to quantify the expected policy robustness across 32 samples.

\begin{table}[t]
\centering
\small
\caption{\textbf{Reasoning performance comparison across model scales (1.5B/3B/7B/8B).} We report mean@32 and pass@32 (\%) on AMC/AIME. \textcolor{red}{Red} marks the best average. BandPO consistently outperforms vanilla GRPO and heuristic variants (\textit{Clip-Higher}).}
\label{tab:main_results}
\arrayrulecolor[rgb]{0.753,0.753,0.753}
\setlength{\tabcolsep}{0pt} 
\begin{tabular*}{\linewidth}{l@{\extracolsep{\fill}}cccc} 
\arrayrulecolor{black}\hline
\multicolumn{1}{c}{\multirow{2}{*}{\textbf{Method}}} & \textbf{AMC2023} & \textbf{AIME2024} & \textbf{AIME2025} & \textbf{Average} \\ 
\arrayrulecolor[rgb]{0.502,0.502,0.502}\cline{2-5}
\multicolumn{1}{c}{} & mean@32/pass@32 & mean@32/pass@32 & mean@32/pass@32 & mean@32/pass@32 \\ 
\arrayrulecolor{black}\hline
\multicolumn{5}{c}{{\cellcolor[rgb]{0.753,0.753,0.753}}\textbf{DeepSeek-R1-Distill-Qwen-1.5B (800 steps)}} \\
GRPO & 72.11 / 94.31 & 18.13 / 39.00 & 21.88 / 38.89 & 37.37 / 57.40 \\
GRPO w/ Clip-Higher & 77.03 / 94.98 & 18.23 / 41.09 & 23.12 / 40.16 & 39.46 / 58.74 \\ 
\arrayrulecolor[rgb]{0.753,0.753,0.753}\hline
GRPO w/ Relaxed $\text{Band}_{\text{KL},0.05}$ & 74.69 / 93.77 & 19.69 / 43.28 & 23.54 / 38.84 & 39.31 / 58.63 \\
GRPO w/ $\text{Band}_{\text{KL},0.05}$ & \textbf{77.34} / \textbf{94.98} & \textbf{20.00} / \textbf{51.80} & \textbf{23.85} / \textbf{40.65} & \textbf{\textcolor{red}{40.40}} / \textbf{\textcolor{red}{62.48}} \\ 
\arrayrulecolor{black}\hline
\multicolumn{5}{c}{{\cellcolor[rgb]{0.753,0.753,0.753}}\textbf{Qwen2.5-3B-Instruct (800 steps)}} \\
GRPO & 45.94 / 77.33 & 3.54 / 11.68 & 3.23 / 8.79 & 17.57 / 32.60 \\
GRPO w/ Clip-Higher & 52.66 / 82.91 & 4.69 / 14.95 & 4.06 / 23.93 & 20.47 / 40.60 \\ 
\arrayrulecolor[rgb]{0.753,0.753,0.753}\hline
GRPO w/ Relaxed $\text{Band}_{\text{KL},0.05}$ & 52.97 / 87.05 & 4.58 / \textbf{15.11} & 4.06 / 21.00 & 20.54 / 41.05 \\
GRPO w/ $\text{Band}_{\text{KL},0.03}$ & 52.81 / \textbf{87.84} & 4.27 / 10.00 & 4.06 / 22.40 & 20.38 / 40.08 \\
GRPO w/ $\text{Band}_{\text{KL},0.10}$ & 51.41 / 84.77 & 3.54 / 14.31 & 6.04 / 20.85 & 20.33 / 39.98 \\
GRPO w/ $\text{Band}_{\text{KL},0.05}$ & \textbf{55.17} / 87.55 & \textbf{4.79} / 14.21 & \textbf{6.04} / \textbf{24.28} & \textbf{\textcolor{red}{22.00}} / \textbf{\textcolor{red}{42.01}} \\ 
\arrayrulecolor{black}\hline
\multicolumn{5}{c}{{\cellcolor[rgb]{0.753,0.753,0.753}}\textbf{DeepSeek-R1-Distill-Qwen-7B (500 steps)}} \\
GRPO & 87.11 / 95.00 & 27.29 / 49.71 & 32.71 / 55.62 & 49.04 / 66.78 \\
GRPO w/ Clip-Higher & 87.50 / 95.00 & 26.77 / 48.11 & 30.83 / 56.96 & 48.37 / 66.69 \\ 
\arrayrulecolor[rgb]{0.753,0.753,0.753}\hline
GRPO w/ Relaxed $\text{Band}_{\text{KL},0.05}$ & 88.58 / 95.00 & 29.69 / 50.78 & 33.44 / 54.52 & 50.57 / 66.77 \\
GRPO w/ $\text{Band}_{\text{KL},0.03}$ & 88.28 / 95.00 & 29.17 / \textbf{51.58} & 32.71 / \textbf{61.46} & 50.05 / \textbf{\textcolor{red}{69.35}} \\
GRPO w/ $\text{Band}_{\text{KL},0.10}$ & 89.69 / 95.00 & 29.79 / 47.19 & 32.29 / 52.03 & 50.59 / 64.74 \\
GRPO w/ $\text{Band}_{\text{KL},0.05}$ & \textbf{89.84} / \textbf{95.00} & \textbf{29.90} / 49.14 & \textbf{34.58} / 57.21 & \textbf{\textcolor{red}{51.44}} / 67.12 \\ 
\arrayrulecolor{black}\hline
\multicolumn{5}{c}{{\cellcolor[rgb]{0.753,0.753,0.753}}\textbf{DeepSeek-R1-Distill-Llama-8B (500 steps)}} \\
GRPO & 85.47 / 94.11 & 23.23 / 46.00 & 23.54 / 54.80 & 44.08 / 64.97 \\
GRPO w/ Clip-Higher & 86.79 / 94.13 & 24.58 / 46.49 & 28.33 / \textbf{61.86} & 46.57 / 67.49 \\ 
\arrayrulecolor[rgb]{0.753,0.753,0.753}\hline
GRPO w/ Relaxed $\text{Band}_{\text{KL},0.05}$ & 86.72 / 94.67 & 25.00 / 49.82 & 29.79 / 53.55 & 47.17 / 66.01 \\
GRPO w/ $\text{Band}_{\text{KL},0.05}$ & \textbf{87.03} / \textbf{95.00} & \textbf{25.31} / \textbf{51.21} & \textbf{29.90} / 57.61 & \textbf{\textcolor{red}{47.41}} / \textbf{\textcolor{red}{67.94}} \\
\arrayrulecolor{black}\hline
\end{tabular*}
\vspace{-1em}
\end{table}

\subsection{Main Results}
As presented in Table~\ref{tab:main_results}, BandPO consistently outperforms all baselines in \textbf{mean@32} across diverse models and datasets, demonstrating superior exploitation capabilities. 
Specifically, BandPO improves mean@32 by at least 2.0 points over GRPO across all settings, with a notable gain of $\sim$10 points on the Qwen2.5-3B AMC2023 task. 
In contrast, baseline methods exhibit significant instability: \textit{GRPO w/ Clip-Higher} suffers from performance regression on the 7B AIME benchmarks (dropping 1--2 points), and the 1.5B GRPO baseline exhibits severe instability, consistently collapsing around training step 340 (verified across multiple runs).
Crucially, BandPO substantially lifts \textbf{pass@32} (e.g., a 28.9\% relative gain on 3B) and consistently surpasses heuristic baselines like Clip-Higher and Relaxed-Band. 
Overall, these results confirm that the probability-aware constraints provided by Band offer a superior exploration-exploitation trade-off compared to heuristic clipping bounds of the probability ratio.

\subsection{Relaxing Band Bounds Degrades Performance}
\label{sec:relaxed_band_analysis}
As illustrated in Figure~\ref{fig:3d-vis-compare}, while BandPO significantly expands the bounds for low-probability actions, it imposes tighter constraints in the high-probability regime compared to fixed asymmetric bounds (DAPO), as highlighted by the green region.
To investigate whether directly expanding the Band bounds yields further improvements, we evaluate \textbf{GRPO w/ Relaxed $\text{Band}_{\text{KL},\text{0.05}}$}, which heuristically aligns the high-probability bound with Clip-Higher. 
Results in Table~\ref{tab:main_results} demonstrate that this heuristic relaxation leads to consistent overall performance degradation.
For larger models (7B/8B), relaxing the bounds causes a slight decline in overall performance ($\sim$0.5 decrease in average mean@32) and a notable drop of approximately 3 points in pass@32 on AIME 2025.
This detrimental effect is exacerbated in smaller models (1.5B/3B); specifically, the 1.5B model suffers a drop of nearly 8 points in pass@32 on AIME 2024, while the 3B model decreases by roughly 3 points in mean@32 on AMC 2023.
These findings underscore that directly relaxing Band to fully cover the Clip-Higher range is detrimental to performance.
This suggests that clipping bounds should not be altered heuristically but should stem from rigorous theoretical derivation, underscoring the importance of BandPO's theoretically grounded formulation.

\subsection{Radius $\delta$ Matters More for Smaller Backbones}
\label{sec:delta_radius}
While canonical clipping typically employs fixed heuristic thresholds (e.g., $\epsilon_{-} = 0.2, \epsilon_{+} = 0.28$) across varying model scales, BandPO consolidates these constraints into a single, theoretically grounded hyperparameter: the trust region radius $\delta$.
To investigate the impact of $\delta$ on performance and stability, we evaluate the KL-instantiated BandPO, varying $\delta \in \{0.03, 0.05, 0.10\}$ on Qwen2.5-3B and Qwen2.5-7B backbones.
As evidenced in Table~\ref{tab:main_results}, setting $\delta=0.05$ consistently yields superior performance in terms of \textbf{mean@32} for both model sizes compared to tighter ($\delta=0.03$) or looser ($\delta=0.10$) constraints.
This finding contradicts the intuition that larger $\delta$ monotonically improves performance. 
Instead, it highlights a critical trade-off: while overly strict constraints limit the policy's ability to learn from effective exploration, excessively loose constraints risk destabilizing the update, leading to suboptimal convergence.
Crucially, our analysis reveals that smaller models exhibit higher sensitivity to the trust-region radius.
For the 3B model, the optimal $\delta=0.05$ outperforms the suboptimal settings by approximately 10\% in mean@32 and 5\% in pass@32.
In contrast, the 7B model demonstrates greater robustness, with performance fluctuations across the three $\delta$ settings confined to a narrow margin of 2--3\%.
This disparity suggests that larger models possess a more resilient optimization landscape, capable of tolerating rougher updates, whereas smaller models require precise trust-region management to prevent performance degradation.
Based on these empirical findings, we recommend $\delta=0.05$ as a robust default starting point for practical implementations of $\text{Band}_{\text{KL},\delta}$.

\begin{figure}[t]
    \setlength{\belowcaptionskip}{-10pt}  
    \centering
    \begin{subfigure}[b]{0.8\linewidth}
        \centering
        \setlength{\abovecaptionskip}{1pt} 
        \setlength{\belowcaptionskip}{2pt} 
        \includegraphics[width=\linewidth]{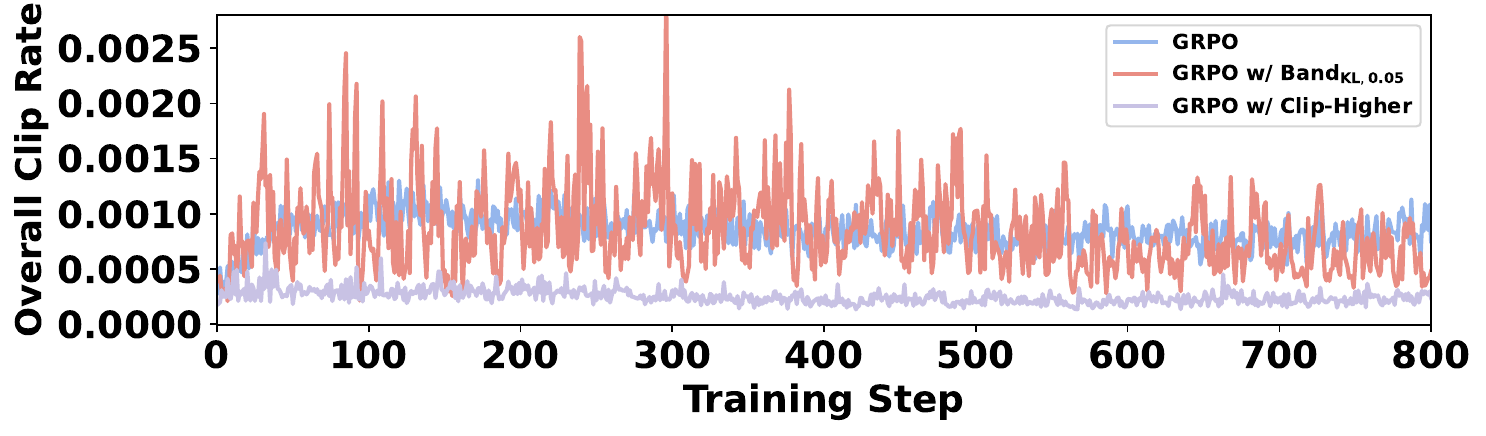}
        \caption{\textbf{Overall clip rate across training steps.}}
        \label{fig:sub1}
    \end{subfigure}
    \begin{subfigure}[b]{0.8\linewidth}
        \centering
        \setlength{\abovecaptionskip}{1pt} 
        \setlength{\belowcaptionskip}{2pt} 
        \includegraphics[width=\linewidth]{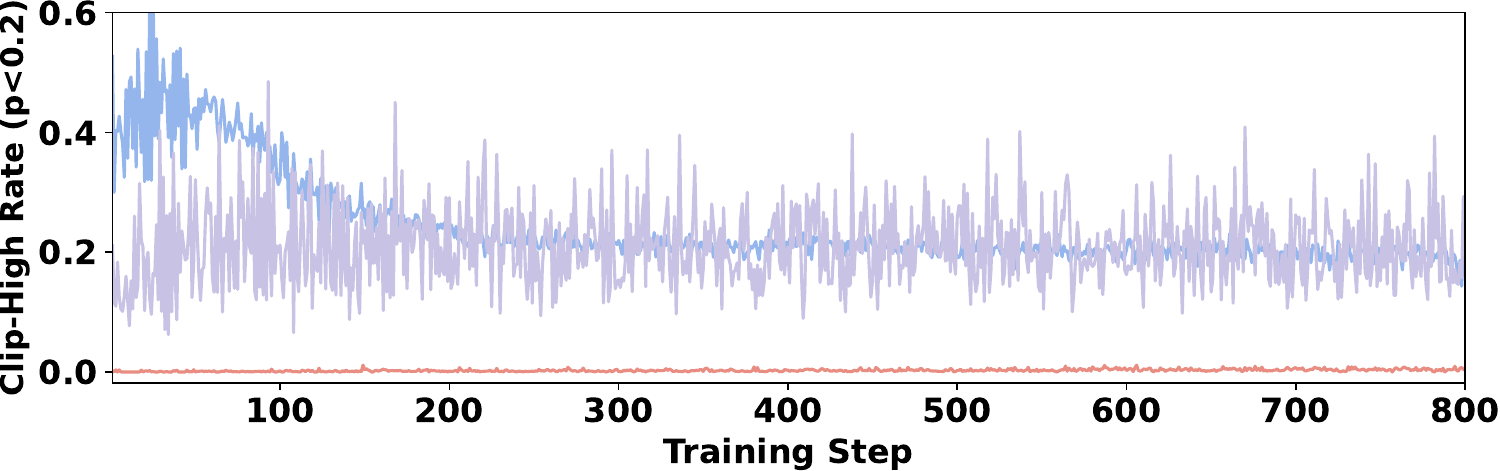}
        \caption{\textbf{Clip-High rate for low-prob tokens ($p < 0.2$).}}
        \label{fig:sub2}
    \end{subfigure}
    \begin{subfigure}[b]{0.8\linewidth}
        \centering
        \setlength{\abovecaptionskip}{1pt} 
        \setlength{\belowcaptionskip}{2pt} 
        \includegraphics[width=\linewidth]{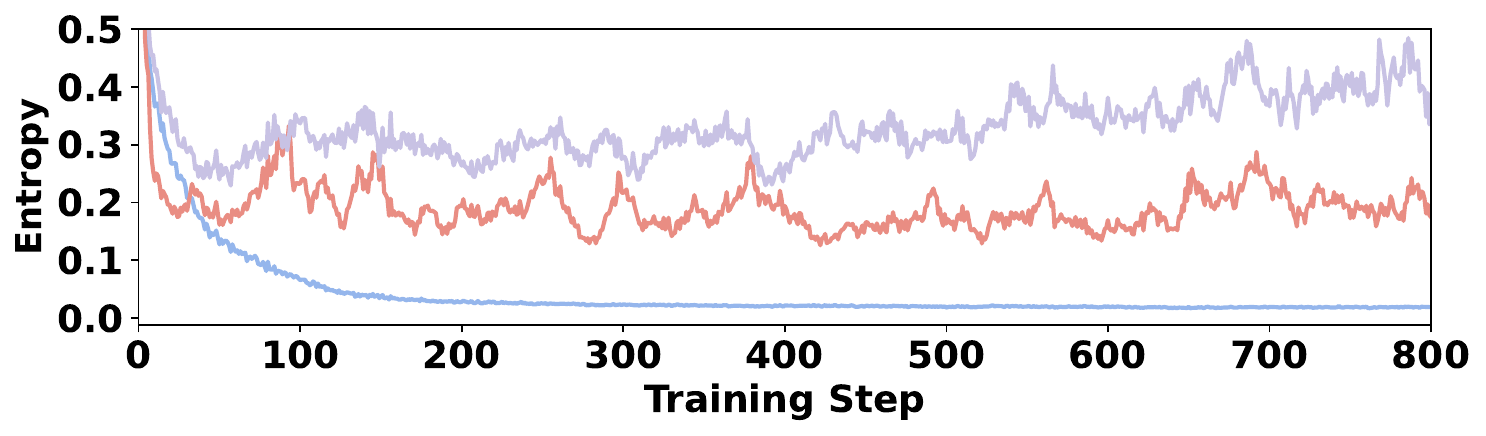}
        \caption{\textbf{Policy entropy evolution.}}
        \label{fig:sub3}
    \end{subfigure}
    \caption{\textbf{Comparison of training dynamics.} (a) Overall clip rate measuring the fraction of clipped tokens relative to total tokens per update. (b) Proportion of clip-high for low-probability tokens ($p < 0.2$) relative to total clipped tokens, identifying erroneous tail-action suppression. (c) Evolution of policy entropy measuring the concentration of action distributions. Our method (red) effectively prevents mode collapse by mitigating the vanishing margin issue in (b).}
    \label{fig:analysis_dynamics}
\end{figure}

\subsection{BandPO Unlocks Exploration for Tail Actions}
\label{sec:analysis_tail}
The theoretical analysis in Section~\ref{subsec:properties} posits that BandPO resolves the exploration bottleneck via the \textbf{Band} operator.
To empirically validate this mechanism, we analyze the training dynamics of the Band constraint (\textbf{GRPO w/ $\text{Band}_{\text{KL}, 0.05}$}) compared to the canonical clipping constraint (\textbf{GRPO}) and the Clip-Higher constraint (\textbf{GRPO w/ Clip-Higher}) on the Qwen2.5-3B-Instruct backbone.
As shown in Figure~\ref{fig:sub1}, the overall clip rate of Band aligns closely with that of canonical clipping, whereas Clip-Higher exhibits a significantly lower rate ($\sim$50\% reduction).
This contrast is expected; as visualized in Figure~\ref{fig:3d-vis-compare}, Band imposes stricter constraints than Clip-Higher (DAPO) in the high-probability regime ($p > 0.8$).

However, aggregate statistics mask a critical distributional shift.
Figure~\ref{fig:sub2} isolates the incidence of upper-bound clipping (clip-high) specifically for low-probability actions ($p < 0.2$).
A distinct contrast emerges: for fixed-bound methods (canonical clipping and Clip-Higher), the upper-bound constraint on tail tokens alone accounts for approximately 20\% of the total clipping volume.
In the early training phase (first 50 steps) of canonical clipping, this proportion surges to 60\%, corresponding to the rapid entropy collapse observed in Figure~\ref{fig:sub3}.
These findings empirically confirm the bottleneck identified in Section~\ref{sec:bottlenecks}: fixed bounds prematurely suppress positive-advantage updates for tail actions due to the vanishing upward margin.

In sharp contrast, Band reduces the clipping incidence for tail actions to nearly zero. By dynamically expanding the upper bound as $p \to 0$, Band ensures that low-probability actions with positive advantages are not prematurely discarded but are instead preserved to contribute effective gradients. 
Crucially, while Band maintains a high aggregate clip rate comparable to canonical clipping (Figure~\ref{fig:sub1}), its entropy evolution mirrors the robust stability of Clip-Higher (Figure~\ref{fig:sub3}).
By substituting the canonical clipping operator with Band, BandPO effectively averts the early-stage entropy collapse characteristic of GRPO, eventually converging to a mean entropy an order of magnitude higher ($0.2$ vs. $0.02$).
This result yields a pivotal insight: optimization stability relies not merely on reducing the total volume of clipping, but on rigorously redistributing the clipping budget. BandPO succeeds by strategically reallocating the permissible update margin—loosening constraints on tail actions to facilitate exploration while tightening them on head actions to enforce stability.

\section{Discussion}
\label{sec:discussion}
In this work, we introduced BandPO, a principled framework that realigns the canonical proximal clipping mechanism with the geometric constraints of trust regions. By projecting $f$-divergence constraints into dynamic, probability-aware intervals, BandPO resolves the exploration bottleneck for tail strategies while maintaining strict simplex consistency. Below, we discuss the implications of our findings, computational considerations, and future directions.

\paragraph{Limitations.} 
\textbf{Computational Overhead of Numerical Solvers.} 
Unlike the computationally trivial canonical clipping mechanism, BandPO involves solving convex constraint equations, which inherently introduces numerical complexity. 
While our derived closed-form solutions for divergences such as TV and Pearson $\chi^2$ (Prop.~\ref{prop:closedform}) maintain analytical efficiency, divergences lacking closed-form inverses—most notably the KL-divergence—necessitate the use of iterative root-finding algorithms.
Consequently, applying BandPO with KL-divergence relies on numerical solvers, which inevitably incurs additional computational latency and minor approximation errors compared to elementary scalar operations.
To mitigate this in latency-critical deployment scenarios, we highlight that the strict monotonicity of Band bounds permits the pre-computation of high-precision lookup tables, effectively reducing the runtime complexity to $O(1)$ memory access.
\textbf{Static Trust Region Assumptions.} 
Our current framework enforces a global trust region radius $\delta$ across all tokens. This assumption of a uniform trust budget homogenizes the diverse landscape of language generation, overlooking the varying informational value of tokens. 
In practice, routine syntactic connectors and pivotal reasoning leaps likely require distinct stability margins. 
Consequently, a fixed $\delta$ represents a compromise: it may be inadvertently too loose for high-confidence syntax (risking instability) while simultaneously being too tight for complex reasoning steps (stifling deep exploration).

\paragraph{Future Work.}
A natural evolution of this framework is to transition from static to \textit{adaptive} Band operators. 
We plan to investigate methods where the radius $\delta_t$ is dynamically modulated by token-level metrics, such as policy entropy or semantic uncertainty. 
By assigning tighter constraints to low-entropy syntactic transitions and relaxing bounds for high-stakes reasoning steps, we aim to further disentangle the trade-off between optimization stability and the exploration of novel reasoning paths.

\section{Conclusion}
We introduced BandPO, a principled optimization framework that bridges the gap between computationally efficient clipping and rigorous trust region constraints. By projecting $f$-divergence balls into dynamic, probability-aware bounds, BandPO effectively resolves the exploration bottleneck for high-advantage tail actions while maintaining training stability. Extensive experiments across models from 1.5B to 8B demonstrate that BandPO significantly outperforms heuristic baselines and robustly prevents entropy collapse. These results highlight the limitations of static heuristics and demonstrate the efficacy of geometrically grounded constraints in optimizing complex reasoning policies.

\section*{Impact Statement}
This paper presents work whose goal is to advance the field of Machine Learning. There are many potential societal consequences of our work, none of which we feel must be specifically highlighted here.
\clearpage
\bibliographystyle{plainnat}
\bibliography{main}
\clearpage
\beginappendix
\section{Omitted Proofs}
\label{app:proofs}
\subsection{Proof of Lemma~\ref{lem:complement_rescaling}}
\label{app:proof_lemma_scaling}
\begin{proof}
Without loss of generality, we analyze the structural properties of the optimal solution for the maximization formulation in Problem~\eqref{eq:upper_bound_opt}. The derivation for the minimization case is identical due to the shared feasible set.

Let $q \triangleq Q(a)$ be a fixed probability mass assigned to the target action $a$. To maximize the feasible range of $q$ under the divergence constraint, the optimal policy must distribute the remaining probability mass $1-q$ over the complement set $\mathcal{V}' = \mathcal{V} \setminus \{a\}$ such that the divergence contribution is minimized. This sub-problem is formulated as:
\begin{equation}
\min_{\{Q(b)\}_{b \in \mathcal{V}'}} \quad \sum_{b \in \mathcal{V}'} P(b) f\left(\frac{Q(b)}{P(b)}\right) \quad \text{s.t.} \quad \sum_{b \in \mathcal{V}'} Q(b) = 1 - q.
\end{equation}

\textbf{1. Optimality via Symmetry and Convexity.}
Observe that the objective function is a sum of strictly convex terms, and the variables $\{Q(b)\}_{b \in \mathcal{V}'}$ appear symmetrically in the objective (weighted by $P(b)$) and the constraint.
Due to the strict convexity of the generator function $f$, the global minimum of this sub-problem must satisfy the symmetry condition where the likelihood ratio is constant across all complement actions.
Formally, for any $b_i, b_j \in \mathcal{V}'$, optimality requires:
\begin{equation}
\frac{Q(b_i)}{P(b_i)} = \frac{Q(b_j)}{P(b_j)} = c,
\end{equation}
for some scalar $c \ge 0$. If this were not true, one could construct a strictly better solution by averaging the ratios, thereby reducing the strictly convex objective (Jensen's Inequality).
Thus, the optimal structure is necessarily $Q(b) = c P(b)$ for all $b \neq a$.

\textbf{2. Determination of the Scaling Factor.}
Substituting $Q(b) = c P(b)$ into the probability mass constraint $\sum_{b \in \mathcal{V}'} Q(b) = 1 - q$ yields:
\begin{equation}
c \sum_{b \in \mathcal{V}'} P(b) = 1 - q \implies c (1 - P(a)) = 1 - q \implies c = \frac{1 - q}{1 - P(a)}.
\end{equation}
Recall that $q = r P(a)$ where $r$ is the ratio at the target action. This recovers the scaling factor $c(r) = \frac{1 - r P(a)}{1 - P(a)}$.

\textbf{3. Strict Interiority.}
In the context of LLM post-training, the reference policy $\pi_{\text{old}}$ is derived from a Softmax distribution, ensuring $P(v) > 0$ for all $v \in \mathcal{V}$.
Furthermore, the trust-region radius $\delta$ is typically small (e.g., $\delta \ll D_f(\mathbb{I}_a \| P)$), ensuring that the feasible set does not include the degenerate solution where $Q(a)=1$ (which would imply an infinite or maximally large divergence).
Consequently, we strictly have $q < 1$, which implies $c = \frac{1-q}{1-P(a)} > 0$.
Since $P(b) > 0$ and $c > 0$, it follows that $Q(b) > 0$ for all $b$. Thus, the optimal solution strictly resides in the interior of the probability simplex, justifying the validity of the derivation without requiring ad-hoc assumptions.
\end{proof}
\subsection{Proof of Theorem~\ref{thm:scalar_reduction}}
\label{app:proof_thm_scalar}

\begin{proof}
We aim to prove that the scalar function $g_f(p, r)$ is strictly convex and that the roots of $g_f(p, r) = \delta$ uniquely determine the optimal clipping bounds.

\textbf{1. Strict Convexity of $g_f(p, r)$.}
Recall the definition of the scalarized divergence function:
\begin{equation}
g_f(p, r) = p f(r) + (1 - p) f\left( c(r) \right), \quad \text{where } c(r) = \frac{1 - rp}{1 - p}.
\end{equation}

We compute the first and second derivatives with respect to $r$. 
First, note that $\frac{d c(r)}{d r} = \frac{-p}{1-p}$.
The first derivative is:
\begin{equation}
\frac{\partial g_f}{\partial r} = p f'(r) + (1-p) f'(c(r)) \left( \frac{-p}{1-p} \right) = p \left[ f'(r) - f'(c(r)) \right].
\end{equation}

The second derivative is:
\begin{equation}
\frac{\partial^2 g_f}{\partial r^2} = p f''(r) - p f''(c(r)) \left( \frac{-p}{1-p} \right) = p f''(r) + \frac{p^2}{1-p} f''(c(r)).
\end{equation}
Since $f$ is strictly convex by definition, we have $f''(x) > 0$ for all $x$. Given that $p \in (0,1)$, it follows that $\frac{\partial^2 g_f}{\partial r^2} > 0$ for all valid $r$. Thus, $g_f(p, r)$ is strictly convex.

\textbf{2. Global Minimum.}
We examine the critical point where $\frac{\partial g_f}{\partial r} = 0$.
$p [f'(r) - f'(c(r))] = 0 \implies f'(r) = f'(c(r))$. Since $f$ is strictly convex, $f'$ is strictly increasing and injective, implying $r = c(r)$.
Solving $r = \frac{1-rp}{1-p}$ yields $r(1-p) = 1-rp \implies r = 1$.
At $r=1$, we have $g_f(p, 1) = p f(1) + (1-p) f(1) = 0$ (since $f(1)=0$).
Therefore, $g_f$ achieves its unique global minimum of 0 at $r=1$.

\textbf{3. Existence and Uniqueness of Roots.}
Since $g_f(p, r)$ is strictly convex and minimal at $r=1$, it is strictly decreasing on the interval $[0, 1)$ and strictly increasing on $(1, 1/p]$.
For any valid trust region $\delta > 0$:
\begin{itemize}
    \item At the minimum, $g_f(p, 1) = 0 < \delta$.
    \item As $r$ moves away from 1 towards the boundaries (0 or $1/p$), $g_f$ increases strictly. Assuming $\delta$ is within the feasible range of the divergence (which is true for standard small $\delta$), the Intermediate Value Theorem guarantees exactly one root in $[0, 1)$ (denoted $\underline{r}$) and exactly one root in $(1, 1/p]$ (denoted $\overline{r}$).
\end{itemize}
Thus, the inequality constraint $D_f(Q \| P) \le \delta$, which is equivalent to $g_f(p, r) \le \delta$, corresponds precisely to the interval $r \in [\underline{r}, \overline{r}]$. 
Since the objective is to maximize (or minimize) $r$, the solution must lie on the boundary of the feasible interval. Thus, the optimal bounds are strictly defined by the unique roots of the binding equation $g_f(p, r) = \delta$.
\end{proof}
\subsection{Proof of Proposition~\ref{prop:limits} (Asymptotic Behavior)}
\label{app:proof_prop_limits}
\begin{proof}
We investigate the limiting behavior of the optimal bounds $\overline{r}_{f,\delta}(p)$ and $\underline{r}_{f,\delta}(p)$ by analyzing the feasibility set defined by $g_f(p, r) \le \delta$.

\textbf{1. Tail Expansion ($p \to 0^+$).}
Recall the scalarized divergence $g_f(p, r) = p f(r) + (1-p) f(c(r))$ with $c(r) = \frac{1-rp}{1-p}$.
As $p \to 0^+$, for any fixed finite $r$, we have $c(r) \to 1$. Since $f(1)=0$, the complement term $(1-p)f(c(r))$ vanishes. The constraint asymptotically behaves as $p f(r) \le \delta$.
\begin{itemize}[leftmargin=*]
    \item \textbf{Upper Bound ($\overline{r}$):} For $r > 1$, convex $f$ grows superlinearly (or at least linearly for valid divergences). The binding condition implies $f(\overline{r}) \approx \delta / p$. As $p \to 0^+$, $\delta/p \to +\infty$, forcing $f(\overline{r}) \to +\infty$ and consequently $\overline{r} \to +\infty$.
    \item \textbf{Lower Bound ($\underline{r}$):} The lower bound is defined as $\underline{r} = \min \{r \in [0, 1] \mid g_f(p, r) \le \delta\}$.
    As $p \to 0$, the maximum divergence contribution from the target action is bounded by $p f(0)$ (assuming $r=0$). If $\lim_{p\to 0} p f(0) = 0 < \delta$, then for sufficiently small $p$, the inequality $g_f(p, r) \le \delta$ holds for all $r \in [0, 1]$. In this regime, the constraint is inactive, and the minimization is solely determined by the simplex boundary $r \ge 0$. Thus, $\lim_{p \to 0} \underline{r} = 0$.
\end{itemize}

\textbf{2. Head Constriction ($p \to 1^-$).} We analyze the distinct limiting behaviors for the upper and lower bounds separately as $p \to 1^-$.
\begin{itemize}
\item \textbf{Upper Bound ($\overline{r} \to 1$):} By definition, $Q(a) = rP(a) = rp$. Since $Q \in \Delta^V$, we strictly require $Q(a) \le 1$, implying the feasible domain for $r$ is $r \in [0, 1/p]$. As $p \to 1^-$, the feasible upper limit $1/p \to 1$. Since the upper bound must theoretically satisfy $\overline{r} \ge 1$, the domain constraint forcefully squeezes it: $1 \le \lim_{p \to 1^-} \overline{r} \le \lim_{p \to 1^-} (1/p) = 1$. Thus, $\lim_{p \to 1^-} \overline{r} = 1$.

\item \textbf{Lower Bound ($\underline{r} \to r^\ast$):} Unlike the upper bound, the lower bound $\underline{r} \le 1$ operates within $[0, 1]$ and is \textit{not} squeezed by the simplex boundary as $p \to 1^-$. Instead, it converges to a constant $r^\ast$ governed by the asymptotic properties of the $f$-divergence. Let $C_\infty \triangleq \lim_{u \to +\infty} \frac{f(u)}{u}$ denote the asymptotic linear growth rate. As $p \to 1^-$, assuming $\underline{r} < 1$, the complement ratio $u = \frac{1-\underline{r}p}{1-p} \to +\infty$. We can rewrite the complement divergence term from the binding equation as:
\begin{equation}
(1-p) f(u) = (1-p) u \frac{f(u)}{u} = \left( 1 - p + p(1-\underline{r}) \right) \frac{f(u)}{u} \xrightarrow{p \to 1^-} (1-\underline{r}) C_\infty.
\end{equation}
Taking the limit $p \to 1^-$ on both sides of the scalarized binding equation $g_f(p, \underline{r}) = \delta$, we obtain:
\begin{equation}
f(r^\ast) + (1-r^\ast) C_\infty = \delta, \quad \text{where } r^\ast = \lim_{p \to 1^-} \underline{r}_{f,\delta}(p).
\end{equation}
Depending on the specific $f$-divergence employed, solving this exact limit equation yields distinct properties for the lower bound:
\begin{itemize}
    \item \textbf{Total Variation (TV):} The generator $f_{TV}(u) = \frac{1}{2}|u-1|$ yields $C_\infty = \lim_{u \to +\infty} \frac{u-1}{2u} = \frac{1}{2}$. The limit equation becomes $\frac{1}{2}(1-r^\ast) + \frac{1}{2}(1-r^\ast) = \delta \implies 1-r^\ast = \delta$. Thus, $r^\ast = \max(1-\delta, 0)$.
    \item \textbf{Pearson $\chi^2$ Divergence:} The generator $f_{\chi^2}(u) = (u-1)^2$ grows super-linearly, yielding $C_\infty = +\infty$. To maintain a finite budget $\delta$, the term $(1-r^\ast) \cdot (+\infty)$ forces $1-r^\ast = 0$. Thus, $r^\ast = 1$.
    \item \textbf{KL Divergence (TRPO/PPO):} The standard generator $f_{KL}(u) = -\log u + u - 1$ yields $C_\infty = \lim_{u \to +\infty} \frac{-\log u + u - 1}{u} = 1$. The limit equation becomes $(-\log r^\ast + r^\ast - 1) + 1 \cdot (1-r^\ast) = \delta \implies -\log r^\ast = \delta$. Thus, $r^\ast = e^{-\delta}$.
\end{itemize}
This theoretical derivation perfectly aligns with the empirical lower bounds observed at $p \to 1$ in Figure 2.
\end{itemize}
This completes the proof. \hfill $\square$
\end{proof}

\subsection{Proof of Proposition~\ref{prop:monotonicity} (Monotonicity)}
\label{app:proof_prop_monotonicity}
\begin{proof}
We establish the strict monotonicity of the roots of $F(p, r) \triangleq g_f(p, r) - \delta = 0$ using the Implicit Function Theorem. The sensitivity of the root $r$ with respect to $p$ is given by:
\begin{equation}
\frac{d r}{d p} = - \frac{\partial F / \partial p}{\partial F / \partial r}.
\end{equation}

\textbf{1. Positivity of $\frac{\partial F}{\partial p}$ via Bregman Divergence.}
Differentiating $g_f$ with respect to $p$, and noting that $\frac{\partial c}{\partial p} = \frac{1-r}{(1-p)^2}$, we derive:
\begin{align}
\frac{\partial F}{\partial p} &= f(r) - f(c) + (1-p) f'(c) \frac{\partial c}{\partial p} \nonumber \\
&= f(r) - f(c) + f'(c) \frac{1-r}{1-p}.
\end{align}
Using the algebraic identity $c - r = \frac{1-rp}{1-p} - r = \frac{1-r}{1-p}$, we substitute this into the expression:
\begin{equation}
\label{eq:bregman_identity}
\frac{\partial F}{\partial p} = f(r) - f(c) - f'(c)(r - c) \equiv D_f(r, c),
\end{equation}

where $D_f(r, c)$ is the \textbf{Bregman divergence} generated by the strictly convex function $f$. 
By the strict convexity of $f$, the Bregman divergence is strictly positive for all $r \neq c$. Since $r \neq c$ holds strictly for any $r \neq 1$ (which implies $p < 1$ and $\delta > 0$), we conclude that $\frac{\partial F}{\partial p} > 0$ always.

\textbf{2. Sign Analysis of $\frac{\partial F}{\partial r}$.}
Differentiation with respect to $r$ yields:
\begin{equation}
\frac{\partial F}{\partial r} = p f'(r) + (1-p)f'(c)\left(\frac{-p}{1-p}\right) = p [f'(r) - f'(c)].
\end{equation}
Since $f$ is strictly convex, $f'$ is strictly increasing. We analyze the two roots:
\begin{itemize}[leftmargin=*]
    \item \textbf{Upper Bound ($\overline{r} > 1$):} Here $r > 1$. The complement ratio satisfies $c = \frac{1-rp}{1-p} < 1$ (since $r>1 \implies 1-rp < 1-p$). Thus $r > c$, implying $f'(r) > f'(c)$. Hence, $\frac{\partial F}{\partial r} \big|_{\overline{r}} > 0$.
    \item \textbf{Lower Bound ($\underline{r} < 1$):} Here $r < 1$. It follows that $c > 1$. Thus $r < c$, implying $f'(r) < f'(c)$. Hence, $\frac{\partial F}{\partial r} \big|_{\underline{r}} < 0$.
\end{itemize}

\textbf{3. Conclusion.}
Applying the implicit derivative signs:
\begin{align}
\frac{d \overline{r}}{d p} &= - \frac{(+)}{(+)} < 0 \implies \text{Strictly Decreasing}, \\
\frac{d \underline{r}}{d p} &= - \frac{(+)}{(-)} > 0 \implies \text{Strictly Increasing}.
\end{align}
This completes the proof.
\end{proof}
\subsection{Proof of Proposition~\ref{prop:saturation} (Constraint Activity and Saturation)}
\label{app:proof_prop_saturation}

\begin{proof}
We analyze the existence of interior roots for the binding equation $g_f(p, r) = \delta$ relative to the simplex boundaries. Recall from Theorem~\ref{thm:scalar_reduction} that $g_f(p, r)$ is strictly convex with a global minimum $g_f(p, 1) = 0$.

\paragraph{Upper Bound Saturation.}
Consider the function $g_f(p, r)$ on the interval $[1, r_{\max}]$, where $r_{\max} = 1/p$.
Since $g_f$ is strictly increasing for $r > 1$, the maximum divergence is attained at the boundary: $g_{\text{peak}} \triangleq g_f(p, r_{\max})$.
\begin{itemize}[leftmargin=*]
    \item \textbf{Inactive Regime ($g_{\text{peak}} \le \delta$):} 
    If the maximal divergence is within the trust region, then $g_f(p, r) \le g_{\text{peak}} \le \delta$ holds for all $r \in [1, r_{\max}]$. The constraint is inactive. Since the objective in Eq.~\eqref{eq:upper_bound_opt} is to maximize the ratio, the optimal solution saturates at the boundary: $\overline{r}_{f,\delta}(p) = r_{\max}$. In this case, no interior root exists for $g_f(p, r) = \delta$ (unless $\delta = g_{\text{peak}}$, where the root is the boundary itself).
    
    \item \textbf{Active Regime ($g_{\text{peak}} > \delta$):}
    Here we have $g_f(p, 1) = 0 < \delta$ and $g_f(p, r_{\max}) > \delta$. Since $g_f$ is continuous on $[1, r_{\max}]$, the Intermediate Value Theorem guarantees the existence of a root $r^{\dagger} \in (1, r_{\max})$ such that $g_f(p, r^{\dagger}) = \delta$. Furthermore, due to the strict monotonicity of $g_f$ on this interval, this root is unique.
\end{itemize}

\paragraph{Lower Bound Saturation.}
Symmetrically, consider the interval $[r_{\min}, 1]$ where $r_{\min}=0$. $g_f(p, r)$ is strictly decreasing on this interval.
\begin{itemize}[leftmargin=*]
    \item \textbf{Inactive Regime ($g_f(p, r_{\min}) \le \delta$):} The constraint $g_f(p, r) \le \delta$ holds for all $r \in [0, 1]$. Minimizing $r$ leads to saturation at the boundary: $\underline{r}_{f,\delta}(p) = 0$.
    \item \textbf{Active Regime ($g_f(p, r_{\min}) > \delta$):}
    Since $g_f(p, 1) = 0 < \delta$ and $g_f(p, 0) > \delta$, there exists a unique root in $(0, 1)$.
\end{itemize}
Combining these cases yields the formulation in Proposition~\ref{prop:saturation}.
\end{proof}
\subsection{Numerical Implementation Details}
\label{app:numerical_solver}

This section details the numerical procedures for solving the Band bounds in the general case where closed-form solutions are unavailable (e.g., KL divergence).

\subsubsection{Convergence Guarantees}
\label{app:solver_proofs}

We first establish the theoretical basis for the global convergence of bracketed root-finding methods on the scalarized constraint function.

\begin{lemma}[Strict Monotonicity on Intervals]
\label{lem:gf_monotone_intervals}
Let $f: \mathbb{R}_+ \to \mathbb{R}$ be strictly convex with a unique global minimum at $f(1)=0$. The scalarized constraint function
\begin{equation}
    g_f(p, r) \triangleq p f(r) + (1-p) f\left( \frac{1 - rp}{1 - p} \right)
\end{equation}
is strictly convex with respect to $r$. Specifically, $g_f(p, r)$ is strictly decreasing on the interval $[r_{\min}, 1]$ and strictly increasing on $[1, r_{\max}]$.
\end{lemma}

\begin{proof}
The strict convexity of $g_f$ follows directly from the strict convexity of $f$ and the linearity of the argument transformation $c(r)$ (as shown in Theorem~\ref{thm:scalar_reduction}).
Since $g_f(p, r)$ is strictly convex and differentiable on its domain, and attains its global minimum $g_f(p, 1) = 0$, the gradient $\nabla_r g_f(p, r)$ must be strictly negative for $r < 1$ and strictly positive for $r > 1$.
This strictly implies the monotonicity properties required for root isolation.
\end{proof}

\begin{theorem}[Global Convergence of Bisection]
\label{thm:bisection_global}
Consider the active regime where $\min(g_f(p, r_{\min}), g_f(p, r_{\max})) > \delta$. The equation $g_f(p, r) = \delta$ possesses exactly two unique roots: a lower bound $\underline{r} \in (r_{\min}, 1)$ and an upper bound $\overline{r} \in (1, r_{\max})$.
Applying the bisection method on the brackets $[r_{\min}, 1]$ and $[1, r_{\max}]$ guarantees linear convergence to $\underline{r}$ and $\overline{r}$, respectively, to any arbitrary precision $\epsilon$.
\end{theorem}

\begin{proof}
By Lemma~\ref{lem:gf_monotone_intervals}, $g_f(p, r)$ is continuous and strictly monotonic on each identified bracket.
In the active regime, the function values at the bracket endpoints strictly enclose the target value $\delta$ (since $g_f(p, 1) = 0 < \delta$ and the boundary values exceed $\delta$).
The Intermediate Value Theorem (IVT) guarantees the existence of a root in each interval, and strict monotonicity guarantees its uniqueness.
Standard convergence analysis of the bisection method applies directly.
\end{proof}

\subsubsection{Standard Bisection Solver Algorithms}
\label{app:solver_bisection_alg}

We present the standard bisection procedure to solve $g_f(p, r) - \delta = 0$.
Let $\epsilon$ be the convergence tolerance (e.g., $10^{-6}$).

\paragraph{Upper Bound Solver ($\overline{r}$).}
The objective function $h(r) \triangleq g_f(p, r) - \delta$ is strictly \textbf{increasing} on the interval $[1, r_{\max}]$.
\begin{itemize}[leftmargin=1.5em, nosep]
    \item \textbf{Initialize:} $L \leftarrow 1$, $R \leftarrow 1/p$.
    \item \textbf{Iterate} until $|R - L| < \epsilon$:
    \begin{equation}
        m \leftarrow \frac{L + R}{2}; \quad
        \text{if } g_f(p, m) \le \delta, \text{ set } L \leftarrow m; \text{ else } R \leftarrow m.
    \end{equation}
    \item \textbf{Return:} $L$ (approximating $\overline{r}_{f,\delta}$).
\end{itemize}

\paragraph{Lower Bound Solver ($\underline{r}$).}
The objective function is strictly \textbf{decreasing} on the interval $[0, 1]$. Note the reversed conditional update logic compared to the upper bound.
\begin{itemize}[leftmargin=1.5em, nosep]
    \item \textbf{Initialize:} $L \leftarrow 0$, $R \leftarrow 1$.
    \item \textbf{Iterate} until $|R - L| < \epsilon$:
    \begin{equation}
        m \leftarrow \frac{L + R}{2}; \quad
        \text{if } g_f(p, m) \le \delta, \text{ set } R \leftarrow m; \text{ else } L \leftarrow m.
    \end{equation}%
    \item \textbf{Return:} $R$ (approximating $\underline{r}_{f,\delta}$).
\end{itemize}



\subsubsection{Instantiation: Forward KL Divergence}
We illustrate the solver with the forward KL divergence, which is the standard metric for TRPO. This trust region is recovered by the generator function $f_{KL}(u) = -\log u + u - 1$. Substituting $f_{KL}$ into the scalarized form yields the binding equation:
\begin{equation}
    p (-\log r + r - 1) + (1 - p) \Big( -\log c(r) + c(r) - 1 \Big) = \delta,
\end{equation}
where $c(r) = \frac{1-rp}{1-p}$ is the complement scaling factor. Simplifying the linear terms (which exactly cancel out to zero: $pr - p + 1 - pr - 1 + p = 0$), this reduces to a transcendental equation strictly dominated by the logarithmic terms:
\begin{equation}
    -p \log r - (1 - p) \log \left( \frac{1 - rp}{1 - p} \right) = \delta.
\end{equation}
As this equation involves $r$ logarithmically both inside and outside the complement scaling argument, it admits no closed-form solution in terms of elementary functions. Therefore, the generic numerical solver described in Appendix A.6.2 is necessary to compute the exact Band bounds for KL-based trust regions.

\subsection{Proof of Proposition}
\label{app:proof_prop_closedform}

\begin{proof}
Recall from Theorem~\ref{thm:scalar_reduction} that the exact bounds $\overline{r}_{f,\delta}(p)$ and $\underline{r}_{f,\delta}(p)$ correspond to the roots of the scalar binding equation:
\begin{equation}
\label{eq:proof_binding_eq}
g_f(p, r) \triangleq p f(r) + (1 - p) f\left( \frac{1 - r p}{1 - p} \right) = \delta.
\end{equation}
We substitute the specific generator functions for Total Variation and Pearson $\chi^2$ to derive the solutions.

\paragraph{1. Total Variation (TV).}
The generator function is defined as $f_{\mathrm{TV}}(u) = \frac{1}{2}|u - 1|$. 
Substituting this into Eq.~\eqref{eq:proof_binding_eq} gives:
\begin{equation}
g_{\mathrm{TV}}(p, r) = \frac{p}{2} |r - 1| + \frac{1 - p}{2} \left| \frac{1 - rp}{1 - p} - 1 \right| = \delta.
\end{equation}
We simplify the complement term. Note that:
\begin{equation}
\frac{1 - rp}{1 - p} - 1 = \frac{1 - rp - (1 - p)}{1 - p} = \frac{p(1 - r)}{1 - p}.
\end{equation}
Thus, the second term in the binding equation simplifies to:
\begin{equation}
\frac{1 - p}{2} \left| \frac{p(1 - r)}{1 - p} \right| = \frac{1 - p}{2} \cdot \frac{p}{1 - p} |1 - r| = \frac{p}{2} |r - 1|.
\end{equation}
Substituting this back yields a unified expression:
\begin{equation}
g_{\mathrm{TV}}(p, r) = \frac{p}{2} |r - 1| + \frac{p}{2} |r - 1| = p |r - 1|.
\end{equation}
Setting $g_{\mathrm{TV}}(p, r) = \delta$, we solve for $r$:
\begin{equation}
|r - 1| = \frac{\delta}{p} \implies r = 1 \pm \frac{\delta}{p}.
\end{equation}

\paragraph{2. Pearson $\chi^2$ Divergence.}
The generator function is $f_{\chi^2}(u) = (u - 1)^2$. 
Substituting this into Eq.~\eqref{eq:proof_binding_eq}:
\begin{equation}
g_{\chi^2}(p, r) = p (r - 1)^2 + (1 - p) \left( \frac{1 - rp}{1 - p} - 1 \right)^2 = \delta.
\end{equation}
Using the same algebraic simplification for the complement term as in the TV case:
\begin{equation}
\left( \frac{1 - rp}{1 - p} - 1 \right)^2 = \left( \frac{p(1 - r)}{1 - p} \right)^2 = \frac{p^2}{(1 - p)^2} (r - 1)^2.
\end{equation}
Substituting this back into $g_{\chi^2}(p, r)$:
\begin{align}
g_{\chi^2}(p, r) &= p (r - 1)^2 + (1 - p) \frac{p^2}{(1 - p)^2} (r - 1)^2 \\
&= p (r - 1)^2 \left[ 1 + \frac{p}{1 - p} \right] \\
&= p (r - 1)^2 \left[ \frac{1 - p + p}{1 - p} \right] \\
&= \frac{p}{1 - p} (r - 1)^2.
\end{align}
Setting $g_{\chi^2}(p, r) = \delta$ and solving for $r$:
\begin{equation}
(r - 1)^2 = \delta \frac{1 - p}{p} \implies r = 1 \pm \sqrt{\frac{\delta (1 - p)}{p}}.
\end{equation}

\paragraph{Practical Implementation (Simplex Constraints).} 
The theoretical bounds derived above assume the trust region lies entirely within the simplex. In practice, to ensure global feasibility (i.e., $Q(a) \in [0,1]$), the effective bounds must be clamped. Applying the saturation logic from Proposition~\ref{prop:saturation}, the final operational bounds are:
\begin{equation}
\label{eq:clamp_final_bounds}
\overline{r}_{f,\delta}^{\star}(p) \triangleq \min\!\left(\overline{r}_{f,\delta}(p), \; \frac{1}{p}\right),
\qquad
\underline{r}_{f,\delta}^{\star}(p) \triangleq \max\!\left(\underline{r}_{f,\delta}(p), \; 0\right).
\end{equation}
\end{proof}
\end{document}